\documentclass[11pt]{article}

\usepackage{amsmath} 
\usepackage{amsthm} 
\usepackage{amsfonts} 
\usepackage{amssymb} 
\usepackage{mathtools} 
\usepackage{thm-restate} 

\usepackage{xcolor} 
\usepackage{bm} 

\usepackage{array} 
\usepackage[inline]{enumitem} 

\usepackage{microtype} 
\usepackage{lmodern} 

\usepackage{tikz} 
\usetikzlibrary{cd} 
\usetikzlibrary{calc} 

\usepackage[in]{fullpage} 

\usepackage[hypcap=true]{subcaption} 

\usepackage[%
  plainpages=false,
  pdfpagelabels,
  colorlinks,
  citecolor=black,
  linkcolor=black,
  urlcolor=black,
  filecolor=black,
  bookmarksopen=false
]{hyperref} 
\usepackage[all]{hypcap} 
\hypersetup{hypertexnames=false}

\newtheoremstyle{thmstyle}
  {\medskipamount}
  {\smallskipamount}
  {\slshape}
  {0pt}
  {\bfseries}
  {.}
  { }
  {\thmname{#1}\thmnumber{ #2}{\normalfont\thmnote{ (#3)}}}
  
\newtheoremstyle{plainstyle}
  {\medskipamount}
  {\smallskipamount}
  {\rmfamily}
  {0pt}
  {\bfseries}
  {.}
  { }
  {\thmname{#1}\thmnumber{ #2}{\normalfont\thmnote{ (#3)}}}

\theoremstyle{thmstyle}
\newtheorem{theorem}{Theorem}[section]
\newtheorem{lemma}[theorem]{Lemma}

\theoremstyle{plainstyle}
\newtheorem{definition}[theorem]{Definition}

\newlist{enumdef}{enumerate}{1}
\setlist[enumdef]{before={\leavevmode}, label={\arabic*.}, ref={\thetheorem.\arabic*}}

\setlist[enumerate]{label={\roman*.}, ref={(\roman*)}} 

\makeatletter
\newcommand{\biggg}{\bBigg@\thr@@}
\def\bigggl{\mathopen\biggg}

\def\bigggr{\mathclose\biggg}
\makeatother

\numberwithin{equation}{section} 

\let\epsilon\varepsilon

\newcommand{\rn}{\bm}
\newcommand{\df}{\stackrel{\text{def}}{=}}

\newcommand{\comp}{\mathbin{\circ}}
\newcommand{\rest}{\mathord{\vert}}

\newcommand{\Floor}[1]{\left\lfloor#1\right\rfloor}

\newcommand{\given}[1][]{\mathrel{#1\vert}}

\DeclareMathOperator{\im}{im}
\DeclareMathOperator{\PAC}{PAC}
\DeclareMathOperator{\gVC}{gVC}
\DeclareMathOperator{\VC}{VC}

\DeclareMathOperator{\VCN}{VCN}
\DeclareMathOperator{\ev}{ev}
\newcommand{\kpart}[1][k]{#1\operatorname{-part}}

\newcommand{\EE}{\mathbb{E}}
\newcommand{\NN}{\mathbb{N}}
\newcommand{\PP}{\mathbb{P}}
\newcommand{\RR}{\mathbb{R}}

\newcommand{\cA}{\mathcal{A}}
\newcommand{\cB}{\mathcal{B}}
\newcommand{\cF}{\mathcal{F}}
\newcommand{\cH}{\mathcal{H}}
\newcommand{\cP}{\mathcal{P}}
\newcommand{\cS}{\mathcal{S}}

\newcommand{\Cervonenkis}{\v{C}ervonenkis}
\newcommand{\Szemeredi}{Szemer\'{e}di}

\title{High-arity Sample Compression}
\author{%
  Leonardo N.~Coregliano
  \and
  William Opich
}
\date{\today}

\begin{document}
\maketitle

\begin{abstract}
  Recently, a series of works have started studying variations of concepts from learning theory for product spaces, which can be
  collected under the name high-arity learning theory. In this work, we consider a high-arity variant of sample compression
  schemes and we prove that the existence of a high-arity sample compression scheme of non-trivial quality implies high-arity PAC
  learnability. 
\end{abstract}

\section{Introduction}

Given a family $\cH$ of subsets of some fixed $X$, can we approximately recover a hidden set $F\in\cH$ from its behavior on a
finite random i.i.d.\ sample from a distribution on $X$? Classes that admit such a procedure are called probably approximately
correctly (PAC) learnable.

An a priori unrelated question about the family $\cH$ is: given a sample $x_1,\ldots,x_m\in X$ along with the behavior of
some $F\in\cH$ on $x_1,\ldots,x_m$, is it possible to compress this information by selecting a small subsample
$x_{i_1},\ldots,x_{i_s}$ so that it is possible to perfectly retrieve $F$ from its behavior only on the subsample
$x_{i_1},\ldots,x_{i_m}$ along with very little extra information? Procedures that achieve this compression and retrieval are
called sample compression schemes.

As it turns out~\cite{LW86,DMY16,MY16}, PAC learnability of a class is equivalent to the existence of a sample
compression scheme that is non-trivial in the sense that the subsample (along with the extra information) is of size $o(m)$. In
turn, sample compression also enjoys a boosting phenomenon in the sense that when a non-trivial sample compression scheme
exists, then there is also a sample compression scheme of size $O(\ln m\cdot\ln\ln m)$. This puts sample compression into the
list of the many learning theory notions that are equivalent to finiteness of the Vapnik--\Cervonenkis\ ($\VC$) dimension (see
e.g.~\cite[Theorem~6.7]{SB14} for a non-comprehensive list as starting point).

Recently, there has been a series of works~\cite{Kob15,TT15,LM19a,LM19b,CM24,CM25a,CM25b,HMS26} studying variations of concepts
from learning theory for product spaces, which can be collected under the name high-arity learning theory. The overall
philosophy of high-arity learning theory is that when we replace the underlying set $X$ by a product space
$X_1\times\cdots\times X_k$ and require the learning theory concept to respect the product space structure, then more classes
satisfy the high-arity learning concept than the classic learning concept. To illustrate this high-arity learning theory
principle, introduce some of the high-arity concepts and frame the current work, let us briefly describe some of the results
that appear in the works that are more closely related to the current one (see Section~\ref{subsec:other} for a brief
description of the others):
\begin{itemize}
\item In~\cite{LM19a,LM19b}, Livni--Mansour introduce the high-arity uniform convergence property of a class $\cH$ of subsets of
  $X_1\times\cdots\times X_k$ as the property that for all distributions $\mu_i$ over $X_i$, if we sample $m=m(\epsilon,\delta)$
  i.i.d.\ points $\rn{x}_1^i,\ldots,\rn{x}_m^i$ from each $\mu_i$, then with probability at least $1-\delta$, the sample will be
  $\epsilon$-representative for $\cH$ in the sense that for every $F\in\cH$, the proportion of $k$-tuples of the form
  $(\rn{x}_{j_1}^1,\ldots,\rn{x}_{j_k}^k)$ that belong to $F$ is $\epsilon$-close to the $\bigotimes_{i=1}^k \mu_i$-measure of
  $F$.

  Livni--Mansour then introduced the graph-$\VC$ dimension ($\gVC$) of such a class $\cH$ as the maximum $\VC$ dimension of the
  classes of slices of $\cH$, that is, once we fix values for $k-1$ coordinates, this produces a slice class of sets on the last
  coordinate, of which one can compute the $\VC$ dimension and maximize over the choice of fixed values.

  Finally, they proved that high-arity uniform convergence property is equivalent to finiteness of $\gVC$ dimension and further
  equivalent to the existence of discriminating algorithms in terms of IPM~distance.
\item In~\cite{CM24,CM25a}, the first author and Malliaris introduced high-arity PAC learnability and proved that it is also
  equivalent to finiteness of $\gVC$ in the binary label case, and in the more general setting of finite labels is equivalent to
  an analogous dimension, which they called the Vapnik--\Cervonenkis--Natarajan $k$-dimension ($\VCN_k$), which is defined
  analogously to $\gVC$ as the maximum Natarajan dimension of the slices. This work also identified a phenomenon exclusive to
  high-arity theory: the presence of two variant of each high-arity concept called the partite and non-partite.

  For partite $k$-PAC learnability, we have a class $\cH$ of functions of the form $X_1\times\cdots\times X_k\to Y$ and we want
  an algorithm $\cA$ such that for all distributions $\mu_i$ over $X_i$, if we sample $m=m(\epsilon,\delta)$ i.i.d.\ points
  $\rn{x}_1^i,\ldots,\rn{x}_m^i$ from each $\mu_i$ and give $\cA$ these $k\cdot m$ points along with the $m^k$ values
  $F(\rn{x}_{j_1}^1,\ldots,\rn{x}_{j_k}^k)$ on every $k$-tuple $(j_1,\ldots,j_k)$, then with probability at least $1-\delta$,
  the algorithm $\cA$ outputs some $H$ that is $\epsilon$-close in $\bigotimes_{i=1}^k \mu_i$-measure to $F$ (see
  Definition~\ref{def:part} for a formal statement).

  For non-partite $k$-PAC learnability, we instead assume that $X_1 = \cdots = X_k$ and we only require $\cA$ to handle the case
  in which $\mu_1 = \cdots = \mu_k$ (see Definition~\ref{def:nonpart} for a formal statement). For a non-partite class $\cH$, we
  can certainly interpret it as a partite class $\cH^{\kpart}$ by simply ignoring the fact that $X_1 = \cdots = X_k$ and it is
  easy to see that partite $k$-PAC learnability of $\cH^{\kpart}$ implies non-partite $k$-PAC learnability, but one of the main
  results of~\cite{CM24} is that, provided the label set $Y$ is finite, these are in fact equivalent, and both characterized by
  finiteness of $\VCN_k$ dimension.

  This work also introduced a (quite technical) notion of high-arity agnostic learnability based on exchangeability theory,
  which we will briefly discuss in Section~\ref{subsec:approxag}.
\item In~\cite{HMS26}, Holzman--Moran--Shlimovich showed that partite $k$-PAC learnability can be upgraded to handle measures
  that are not product measures, but are in a sense not far from them. More specifically, suppose $\cP$ is a collection of
  distributions on $X_1\times\cdots\times X_k$ that are uniformly absolutely continuous with respect to their marginals in the
  sense that for every $\alpha > 0$, there exists $\beta > 0$ such that for every $\mu\in\cP$ and every measurable $B\subseteq
  X_1\times\cdots\times X_k$ with $\mu(B)\geq\alpha$, then $(\bigotimes_{i=1}^k \mu_i)(B)\geq\beta$, where $\mu_i$ is the
  marginal of $\mu$ on $X_i$. In this work, finiteness of $\gVC$ dimension (which is given the name ``linear $\VC$ dimension''
  here) is equivalent to learnability (in the usual classic PAC sense) versus all collections of distributions that are
  uniformly absolutely continuous with respect to their marginals (but, as expected, the learning guarantee depends on the
  quality of the modulus of uniform absolute continuity, i.e., the function $\beta=\beta(\alpha)$).
\end{itemize}

In this work, we introduce the notion of high-arity sample compression (see Definitions~\ref{def:partcomp}
and~\ref{def:nonpartcomp}) and we prove that both in the partite and non-partite settings, high-arity sample compression implies
high-arity PAC learnability (Theorems~\ref{thm:partSC->PAC} and~\ref{thm:nonpartSC->PAC}). As is typical with high-arity
concepts, the main challenge is the fact that labels in high-arity no longer are i.i.d., which precludes us from applying
concentration estimates such as Chernoff or Bernstein Inequalities. However, we show that both in the partite and non-partite
setting, we can appropriately define a martingale that allows us to use Azuma's Inequality instead. We briefly discuss in
Section~\ref{subsec:converse} the difficulties in proving the converse implication (i.e., that high-arity PAC learnability
implies high-arity sample compression).

\subsection{Other related work}
\label{subsec:other}

High-arity learning is much more complex than classical theory. In particular, even for something as simple as PAC learning, the
big picture is that there should be different levels of learnability depending on the particular information that is received by
the learner. For example, the $\PAC_k$ learnability model introduced by Kobayashi~\cite{Kob15}, in which the learner also
receives oracles for slices of the sample, was shown~\cite{TT15} to be governed by a combinatorial dimension connected to the
model-theoretic notion of $k$-dependence introduced by Shelah~\cite[\S~(H)]{She14}. The same notion of $k$-dependence is what
governs~\cite{CPT19} the asymptotic drop via a high-arity analogue of the Sauer--Shelah--Perles Lemma~\cite{Sau72,She72,Per72}
in the number of possible patterns on a grid of size $m$ from $m^k$ to $m^{k-\Omega(1)}$. In turn, such a drop in the number of
possible patterns was then used to characterize~\cite{CM25b} when one can perform matrix-completion in Netflix-like problems.

In a different direction, there is a strong connection between combinatorial dimensions from learning theory to the tamer
\Szemeredi\ Regularity Lemmas. In particular, the $\VC$ dimension was shown to characterize~\cite{AFN07,LS10} classes of graphs
that admit regularity partitions that are homogeneous and classes of hypergraphs that admit a regularity partition into
polynomially many parts~\cite{FPS19,Ter24b}. On the high-arity side, the $\gVC$ dimension (known under other names such as
slicewise $\VC$ dimension in the combinatorial community) was shown to characterize~\cite{TW22,CT20} classes of hypergraphs that
admit homogeneous (vertex) partitions and classes of hypergraphs that admit weak regularity lemmas into exponentially many
parts~\cite{Ter24a,Ter24b,GSW25}. Finally, the aforementioned $k$-dependence was shown to characterize hypergraph classes that
admit a $0/1$-valued strong regularity lemma~\cite{TW22,CT20}. Most of these connections happen via Haussler packing
property~\cite{Hau95} and a high-arity version of it~\cite{CT20,CM25a}.

\subsection{General notation}

The set of non-negative integers is denoted $\NN$ and the set of positive integers is denoted $\NN_+\df\NN\setminus\{0\}$. For a
finite set $V$ and $k\in\NN$, we let $\binom{V}{k}\df\{U\subseteq V\mid\lvert U\rvert=k\}$ be the set of subsets of $V$ of size
$k$ and we let $(V)_k$ be the set of injective functions of the form $[k]\to V$, where $[k]\df\{1,\ldots,k\}$. Given further
$n\in\NN$, we let $(n)_k\df n(n-1)\cdots(n-k+1)$ denote the falling factorial (so $\lvert(V)_k\rvert=(\lvert V\rvert)_k$). For
the particular case of $([k])_k$, we use the more standard notation $S_k$ (as $([k])_k$ is the set of permutations on $[k]$).
For a measurable space $\Omega=(X,\cB)$, we let $\Pr(\Omega)$ be the set of probability measure on $\Omega$. Countable sets will
always be equipped with the discrete $\sigma$-algebra.

\section{Partite setting}

In this section we cover the partite case, that is, the case in which our hypotheses classes are collections of functions of the
form $F\colon X_1\times\cdots\times X_k\to Y$, $k$-PAC learning notions are with respect to product measures
$\bigotimes_{i=1}^k\mu_i$ and samples are thought of as having $m$ points $x_1^j,\ldots,x_m^j$ from each $X_j$ as well as labels
$F(x_{\alpha_1}^1,\ldots,x_{\alpha_k}^k)\in Y$ for each $k$-tuple $\alpha\in[m]^k$.

We use a simplified version of the notation of~\cite{CM24} (the main difference is that since we do not cover either agnostic or
higher-order variables, which allows us to drastically simplify the notation; however, we briefly comment on how our proofs
adapt to these settings in Section~\ref{subsec:approxag}).

\begin{definition}[Partite setting]\label{def:part}
  Let $k\in\NN_+$, let $\Omega=(\Omega_i)_{i=1}^k$ be a $k$-tuple of non-empty standard Borel spaces $\Omega_i=(X_i,\cB_i)$,
  let $\Lambda=(Y,\cB')$ be a non-empty standard Borel space and let $m\in\NN$.
  \begin{enumdef}
  \item\label{def:part:alpha*} A \emph{$k$-partite unlabeled sample of size $m$} with respect to $\Omega$ is an element of
    $\prod_{i=1}^k X_i^m$. A \emph{$k$-partite labeled sample of size $m$} with respect to $\Omega$ and $\Lambda$ is an element
    $(x,y)$ of $(\prod_{i=1}^k X_i^m)\times Y^{[m]^k}$; we think of the coordinate $y_\alpha$ indexed by some $k$-tuple
    $\alpha\in[m]^k$ as the label on $x_{\alpha_1},\ldots,x_{\alpha_k}$. This concept is formalized as follows: for every
    $\alpha\in[m]^k$, we contra-variantly define maps
    \begin{align*}
      \alpha^*\colon\prod_{i=1}^k X_i^m\to\prod_{i=1}^k X_i, & &
      \alpha^*\colon Y^{[m]^k}\to Y
    \end{align*}
    that select the $k$-tuple indexed by $\alpha$ and its label respectively via
    \begin{align*}
      \alpha^*(x)_i & \df x_{\alpha_i} \qquad \left(x\in\prod_{j=1}^k X_j^m, i\in[k]\right),\\
      \alpha^*(y) & \df y_\alpha, \qquad \bigl(y\in Y^{[m]^k}\bigr).
    \end{align*}
  \item With a slight abuse of notation, we denote by $\Pr(\Omega)$ the set of $k$-tuples $\mu=(\mu_i)_{i=1}^k$, where
    $\mu_i\in\Pr(\Omega_i)$ is a probability measure on $\Omega_i$. Given further $m\in\NN$, we let $\mu^m\df\bigotimes_{i=1}^k
    \mu_i^m$ denote the product probability measure on the set $\prod_{i=1}^k X_i^m$ of $k$-partite unlabeled samples of size
    $m$ that is $m$ copies of each measure $\mu_i$. With a small abuse of notation, we view $\mu^1$ also as a measure over
    $\prod_{i=1}^k X_i$ via the natural identification with $\prod_{i=1}^k X_i^1$.
  \item\label{def:part:F*} A \emph{$k$-partite hypothesis} from $\Omega$ to $\Lambda$ is a measurable function
    $F\colon\prod_{i=1}^k X_i\to Y$. The set of $k$-partite hypothesis from $\Omega$ to $\Lambda$ is denoted
    $\cF_k(\Omega,\Lambda)$.

    Given a $k$-partite hypothesis $F\in\cF_k(\Omega,\Lambda)$ and $m\in\NN$, we define the function $F^*_m\colon\prod_{i=1}^k
    X_i^m\to Y^{[m]^k}$ that when given an unlabeled sample of size $m$, provides the labels of $F$ on the $k$-tuples of the
    sample; in a formula:
    \begin{gather*}
      F^*_m(x)_\alpha \df F\bigl(\alpha^*(x)\bigr)
      \qquad \left(x\in\prod_{i=1}^k X_i^m, \alpha\in [m]^k\right).
    \end{gather*}
  \item A \emph{$k$-partite hypothesis class} is a set $\cH\subseteq\cF_k(\Omega,\Lambda)$ of $k$-partite hypotheses that is
    further equipped with a $\sigma$-algebra such that:
    \begin{itemize}
    \item the evaluation map $\ev\colon\cH\times\prod_{i=1}^k X_i\to Y$ given by $\ev(H,x)\df H(x)$ is measurable;
    \item for every $H\in\cH$, the singleton $\{H\}$ is measurable;
    \item for every Borel space $\Upsilon$ and every measurable set $A\subseteq\cH\times\Upsilon$, the projection of $A$ onto
      $\Upsilon$, i.e., the set
      \begin{equation*}
        \{\upsilon\in\Upsilon \mid \exists H\in\cH, (H,\upsilon)\in A\}
      \end{equation*}
      is universally measurable\footnote{The reader unfamiliar with these technical measurability conditions can simply assume
      that $\cH$ is equipped with a $\sigma$-algebra that makes it a standard Borel space as it will imply this condition. The
      reader that wants to ignore measurability concerns can simply interpret this as ``all probabilities computed must make
      sense''; in fact, this universal measurability condition is required only to make sense of the notion of high-arity
      uniform convergence, which is out of the scope of this paper.}.
    \end{itemize}
  \item A \emph{$k$-partite loss function} over $\Lambda$ is a measurable function $\ell\colon(\prod_{i=1}^k X_i)\times Y\times
    Y\to\RR_{\geq 0}$ that when given a $k$-tuple $x$ and two labels $y$ and $y'$ assigns a value $\ell(x,y,y')$ that is the
    penalty of guessing $y$ on the tuple $x$ when the correct label was $y'$.

    We define
    \begin{gather*}
      \lVert\ell\rVert_\infty \df \sup_{\substack{x\in\prod_{i=1}^k X_i\\y,y'\in Y}} \ell(x,y,y')
    \end{gather*}
    and we say that $\ell$ is \emph{bounded} if $\lVert\ell\rVert_\infty < \infty$.

    Given further $\mu\in\Pr(\Omega)$ and $k$-partite hypotheses $F$ and $H$, the \emph{total loss} of $H$ with respect to $\mu$,
    $F$ and $\ell$ is
    \begin{gather*}
      L_{\mu,F,\ell}(H) \df \EE_{\rn{x}\sim\mu^1}\Bigl[\ell\bigl(\rn{x}, H(\rn{x}), F(\rn{x})\bigr)\Bigr].
    \end{gather*}

    We say that $F$ is \emph{realizable} in a $k$-partite hypothesis class $\cH\subseteq\cF_k(\Omega,\Lambda)$ with respect to
    $\mu$ and $\ell$ if $\inf_{H\in\cH} L_{\mu,F,\ell}(H) = 0$.
  \item For a $k$-partite hypothesis class $\cH'\subseteq\cF_k(\Omega,\Lambda)$, a \emph{($k$-partite) learning algorithm} with
    outputs in $\cH'$ is a measurable function
    \begin{gather*}
      \cA\colon\bigcup_{m\in\NN} \left(\left(\prod_{i=1}^k X_i^m\right)\times Y^{[m]^k}\right) \to \cH'.
    \end{gather*}

    We say that a $k$-partite hypothesis class $\cH\subseteq\cF_k(\Omega,\Lambda)$ is \emph{improperly $k$-PAC learnable} with
    respect to a $k$-partite loss function $\ell\colon(\prod_{i=1}^k X_i)\times Y\times Y\to\RR_{\geq 0}$ if there exists a
    learning algorithm $\cA$ with outputs in some $k$-partite hypothesis class $\cH'$ and a function
    $m^{\PAC}_{\cH,\ell,\cA}\colon(0,1)^2\to\RR_{\geq 0}$ such that for every $\epsilon,\delta\in(0,1)$, every
    $\mu\in\Pr(\Omega)$, every $F\in\cF_k(\Omega,\Lambda)$ that is realizable in $\cH$ with respect to $\mu$ and $\ell$ and
    every integer $m\geq m^{\PAC}_{\cH,\ell,\cA}(\epsilon,\delta)$, we have
    \begin{gather*}
      \PP_{\rn{x}\sim\mu^m}\biggl[
        L_{\mu,F,\ell}\Bigl(\cA\bigl(\rn{x}, F^*_m(\rn{x})\bigr)\Bigr)\leq\epsilon
        \biggr]
      \geq
      1 - \delta.
    \end{gather*}
    In plain English: when given a random sample of size $m$ drawn from $\mu^m$ labeled according to some realizable $F$, the
    algorithm $\cA$ outputs some hypothesis $H$ with total loss at most $\epsilon$ with probability at least $1-\delta$ over the
    randomness of the sample.

    The function $m^{\PAC}_{\cH,\ell,\cA}$ is called the \emph{learning guarantee} of $\cA$ and $\cA$ is called an \emph{improper
    $k$-PAC learner} for $\cH$ with respect to $\ell$.

    The notions of \emph{proper $k$-PAC learnability} and \emph{proper $k$-PAC learner} are defined analogously but requiring
    further that $\cH'=\cH$.
  \item Given a $k$-partite labeled sample $(x,y)\in(\prod_{i=1}^k X_i^m)\times Y^{[m]^k}$ of size $m$, a $k$-partite loss
    function $\ell\colon(\prod_{i=1}^k X_i)\times Y\times Y\to\RR_{\geq 0}$ and a hypothesis $H\in\cF_k(\Omega,\Lambda)$, the
    \emph{empirical loss} (or \emph{empirical risk}) of $H$ with respect to $(x,y)$ and $\ell$ is
    \begin{gather*}
      L_{x,y,\ell}(H) \df
      \begin{dcases*}
        \frac{1}{m^k}\cdot
        \sum_{\alpha\in[m]^k} \ell\Bigl(\alpha^*(x), H^*_m(x)_\alpha, y_\alpha\Bigr),
        & if $m > 0$,
        \\
        0, & if $m=0$.
      \end{dcases*}
    \end{gather*}

    We say that $(x,y)$ is \emph{sample realizable} in a $k$-partite hypothesis class $\cH\subseteq\cF_k(\Omega,\Lambda)$ with
    respect to $\ell$ if $\inf_{H\in\cH} L_{x,y,\ell}(H)$. It is straightforward to check that if $\mu\in\Pr(\Omega)$,
    $F\in\cF_k(\Omega,\Lambda)$ is realizable in $\cH$ with respect to $\mu$ and $\ell$ and $\rn{x}\sim\mu^m$, then with
    probability $1$, we have that $(\rn{x}, F^*_m(\rn{x}))$ is sample realizable in $\cH$ with respect to $\ell$.
  \item\label{def:part:alphasharp} Given a $k$-partite labeled sample $(x,y)\in(\prod_{i=1}^k X_i^m)\times Y^{[m]^k}$ of size
    $m$, a $k$-tuple $\alpha=(\alpha_i)_{i=1}^k$ of injective functions $\alpha_i\colon [n]\to[m]$ induces naturally a
    \emph{subsample} $(\alpha^\#(x),\alpha^\#(y))$ of size $n$ by picking the points and labels corresponding to the coordinates
    indexed by $\alpha_1,\ldots,\alpha_k$ (in their order). This is formalized as follows: we contra-variantly define the maps
    \begin{align*}
      \alpha^\#\colon\prod_{i=1}^k X_i^m\to\prod_{i=1}^k X_i^n, & &
      \alpha^\#\colon Y^{[m]^k}\to Y^{[n]^k}
    \end{align*}
    that select the points and labels according to
    \begin{align*}
      \bigl(\alpha^\#(x)_i\bigr)_v & \df (x_i)_{\alpha_i(v)}
      \qquad \left(x\in\prod_{j=1}^k X_j^m, i\in[k], v\in[n]\right),
      \\
      \bigl(\alpha^\#(y)\bigr)_\beta & \df y_{\alpha_1(\beta_1),\ldots,\alpha_k(\beta_k)}
      \qquad \bigl(y\in Y^{[m]^k}, \beta\in[n]^k\bigr).
    \end{align*}

    With slight abuse of notation, we also contra-variantly define the map
    \begin{gather*}
      \alpha^\#\colon\left(\prod_{i=1}^k X_i^m\right)\times Y^{[m]^k}
      \to
      \left(\prod_{i=1}^k X_i^n\right)\times Y^{[n]^k}
    \end{gather*}
    by letting it act as the previous two maps in the corresponding coordinates. The overload in the notation is justified by
    the fact that it makes the definition of $F^*_m$ of~\ref{def:part:F*} equivariant in the sense that the
    following diagram is commutative:
    \begin{equation*}
      \begin{tikzcd}
        \prod_{i=1}^k X_i^m
        \arrow[r, "F^*_m"]
        \arrow[d, "\alpha^\#"']
        &
        Y^{[m]^k}
        \arrow[d, "\alpha^\#"]
        \\
        \prod_{i=1}^k X_i^n
        \arrow[r, "F^*_n"]
        &
        Y^{[n]^k}
      \end{tikzcd}
    \end{equation*}
    
    The avid reader will note that the maps $\alpha^*$ of~\ref{def:part:alpha*} are essentially the particular case of the above
    when $n=1$.
  \end{enumdef}
\end{definition}

The next definition is that of a high-arity selection/compression scheme. Informally, a selection scheme has to have two
procedures: a selection procedure $\kappa_m$ that takes in a sample of size $m$ and outputs a smaller subsample of size $s_m$
along with some extra information (a header) and a reconstruction procedure $\rho_m$ that takes in the output of $\kappa_m$ and
reconstructs the original sample. The selection scheme is then called a compression scheme when it correctly reconstructs the
original sample (in the sense of zero loss).

\begin{definition}[Partite selection/compression schemes]\label{def:partcomp}
  Let $k\in\NN_+$, let $\Omega=(\Omega_i)_{i=1}^k$ be a $k$-tuple of non-empty standard Borel spaces $\Omega_i=(X_i,\cB_i)$ and
  let $\Lambda=(Y,\cB')$ be a non-empty standard Borel space.
  \begin{enumdef}
  \item A \emph{$k$-partite selection scheme} with outputs in a $k$-partite hypothesis class
    $\cH'\subseteq\cF_k(\Omega,\Lambda)$ is a tuple $\cS=(\sigma,\eta,\rho)$ such that:
    \begin{itemize}
    \item $\sigma=(\sigma_m)_{m\in\NN}$ is a sequence of maps, called \emph{selection maps}, such that each $\sigma_m$ when
      given a sample of size $m$, outputs $k$ tuples of some length $s_m$ that are the indices of the subsample selected.
      Formally, each $\sigma_m$ is a measurable function
      \begin{gather*}
        \sigma_m\colon\left(\prod_{j=1}^k X_j^m\right)\times Y^{[m]^k}\to\prod_{i=1}^k ([m])_{s_m},
      \end{gather*}
      where $s_m\in\NN$ is some integer, which we call the \emph{selection size} (note that we must have $s_m\leq m$ as the
      domain of $\sigma_m$ is never empty).
    \item $\eta=(\eta_m)_{m\in\NN}$ is a sequence of maps, called the \emph{header maps}, that when given a sample of size $m$,
      outputs some extra information, encoded as an element of $[h_m]$ for some fixed $h_m\in\NN$, that will aid in the
      reconstruction of the original sample. Formally, each $\eta_m$ is a measurable function
      \begin{gather*}
        \eta_m\colon\left(\prod_{j=1}^k X_j^m\right)\times Y^{[m]^k}\to [h_m],
      \end{gather*}
      where $h_m\in\NN_+$ is called the \emph{header size} of $\eta_m$.
    \item $\rho=(\rho_m)_{m\in\NN}$ is a sequence of maps, called the \emph{reconstruction maps}, that when given a labeled
      sample of size $s_m$ and header of size $h_m$ (intended to have come from $\sigma_m$ and $\eta_m$), reconstruct a labeled
      sample of size $m$. Formally, each $\rho_m$ is a measurable function
      \begin{gather*}
        \rho_m\colon\left(\prod_{i=1}^k X_i^{s_m}\right)\times Y^{[s_m]^k}\times[h_m]
        \to
        \cH'.
      \end{gather*}
    \end{itemize}

    The compression procedure of $\cS$ is formalized by defining a sequence $\kappa=(\kappa_m)_{m\in\NN}$ of \emph{compression
    maps} of the form
    \begin{gather*}
      \kappa_m\colon\left(\prod_{i=1}^k X_i^m\right)\times Y^{[m]^k}
      \to
      \left(\prod_{i=1}^k X_i^{s_m}\right)\times Y^{[s_m]^k}\times [h_m]
    \end{gather*}
    by
    \begin{gather*}
      \kappa_m(x,y)
      \df
      \bigl(\sigma_m(x,y)^\#(x,y), \eta_m(x,y)\bigr)
      \qquad
      \left(x\in\prod_{i=1}^k X_i^m, y\in Y^{[m]^k}\right).
    \end{gather*}
    (Recall from Definition~\ref{def:part:alphasharp} that since $\sigma_m(x,y)\in\prod_{i=1}^k ([m])_{s_m}$ is a $k$-tuple of
    injections, it contra-variantly induces a map that selects a subsample $\sigma_m(x,y)^\#(x,y)$ of size $s_m$ from the sample
    $(x,y)$ of size $m$ and the map $\kappa_m$ then takes this subsample an appends to it the header $\eta_m(x,y)$.)

    We will use the notation
    \begin{align*}
      \sigma^\cS & \df \sigma = (\sigma_m)_{m\in\NN}, &
      \eta^\cS & \df \eta = (\eta_m)_{m\in\NN}, &
      \rho^\cS & \df \rho = (\rho_m)_{m\in\NN},
      \\
      s^\cS_m & \df s_m, &
      h^\cS_m & \df h_m, &
      \kappa^\cS & \df \kappa = (\kappa_m)_{m\in\NN}.
    \end{align*}

    If $Y$ is finite, the \emph{compression size} and the \emph{compression bitlength} functions of $\cS$ are the functions
    $c_\cS,b_\cS\colon\NN\to\RR_{\geq 0}$ given by
    \begin{align*}
      c_\cS(m) & \df h_m\cdot\lvert Y\rvert^{s_m^k}, &
      b_\cS(m) & \df \log_2\bigl(c_\cS(m)\bigr).
    \end{align*}
  \item A $k$-partite selection scheme $\cS$ with outputs in a $k$-partite hypothesis class $\cH'\subseteq\cF_k(\Omega,\Lambda)$
    is called a \emph{sample compression scheme} for a $k$-partite hypothesis class $\cH\subseteq\cF_k(\Omega,\Lambda)$ with
    respect to a $k$-partite loss function $\ell\colon(\prod_{i=1}^k X_i)\times Y\times Y\to\RR_{\geq 0}$ if for every $m\in\NN$
    and every labeled sample $(x,y)\in(\prod_{i=1}^k X_i^m)\times Y^{[m]^k}$ of size $m$ that is sample realizable with respect
    to $\ell$, we have
    \begin{gather*}
      L_{x,y,\ell}\Bigl(\rho^\cS_m\bigl(\kappa^\cS_m(x,y)\bigr)\Bigr) = 0,
    \end{gather*}
    or in plain English: if we compress $(x,y)$ via $\kappa^\cS$ then reconstruct via $\rho^\cS$, then the result has zero
    empirical loss.

    We say that the sample compression scheme is \emph{proper} if $\cH=\cH'$.
  \end{enumdef}
\end{definition}

Our goal is to prove that the existence of a $k$-partite compression scheme (of essentially non-trivial quality) implies
$k$-PAC learnability. We start with a concentration lemma based on Azuma's Inequality.

\begin{lemma}\label{lem:partconc}
  Let $k\in\NN_+$, let $\Omega=(\Omega_i)_{i=1}^k$ be a $k$-tuple of non-empty standard Borel spaces $\Omega_i=(X_i,\cB_i)$, let
  $\Lambda=(Y,\cB')$ be a non-empty standard Borel space, let $\cS$ be a $k$-partite selection scheme and let
  $\ell\colon\prod_{i=1}^k X_i\times Y\times Y\to\RR_{\geq 0}$ be a bounded $k$-partite loss function.

  Then for every $\epsilon > 0$, every $\mu\in\Pr(\Omega)$, every integer $m\in\NN$ such that
  \begin{gather}\label{eq:partconc:condition}
    \left(1 - \frac{(m-s^\cS_m)^k}{m^k}\right)\cdot\lVert\ell\rVert_\infty < \epsilon,
  \end{gather}
  every $\eta\in [h^\cS_m]$, every $k$-tuple $\sigma=(\sigma_i)_{i=1}^k$ of injections $\sigma_i\colon[s^\cS_m]\to[m]$ and every
  $F\in\cF_k(\Omega,\Lambda)$, we have
  \begin{multline*}
    \PP_{\rn{x}\sim\mu^m}\Biggl[
      L_{\mu,F,\ell}\biggl(\rho^\cS\Bigl(\sigma^\#\bigl(\rn{x},F^*_m(\rn{x})\bigr), \eta\Bigr)\biggr)
      -
      L_{\rn{x}, F^*_m(\rn{x}), \ell}\biggl(\rho^\cS\Bigl(\sigma^\#\bigl(\rn{x},F^*_m(\rn{x})\bigr), \eta\Bigr)\biggr)
      \geq
      \epsilon
      \Biggr]
    \\
    \leq
    \exp\left(
    -\frac{\widetilde{\epsilon}\cdot(m-s^\cS_m)}{2\cdot k\cdot\lVert\ell\rVert_\infty^2}
    \right),    
  \end{multline*}
  where
  \begin{gather*}
    \widetilde{\epsilon}
    \df
    (\epsilon-(1-(m-s^\cS_m)^k/m^k)\cdot\lVert\ell\rVert_\infty)^2
    >
    0.
  \end{gather*}

  In particular, if $s^\cS_m\leq o(m)$ as $m\to\infty$, then
  \begin{multline*}
    \PP_{\rn{x}\sim\mu^m}\Biggl[
      L_{\mu,F,\ell}\biggl(\rho^\cS\Bigl(\sigma^\#\bigl(\rn{x},F^*_m(\rn{x})\bigr), \eta\Bigr)\biggr)
      -
      L_{\rn{x}, F^*_m(\rn{x}), \ell}\biggl(\rho^\cS\Bigl(\sigma^\#\bigl(\rn{x},F^*_m(\rn{x})\bigr), \eta\Bigr)\biggr)
      \geq
      \epsilon
      \Biggr]
    \\
    \leq
    \exp\left(
    -\frac{\epsilon^2}{2\cdot k\cdot\lVert\ell\rVert_\infty^2}\cdot m
    +\frac{\epsilon\cdot(\epsilon + \lVert\ell\rVert_\infty)}{\lVert\ell\rVert_\infty^2}\cdot o(m)
    \right),
  \end{multline*}
  where the $o(m)$ error term is as $m\to\infty$ with $k$ fixed, is uniform in $\epsilon$, and has the same uniformity in
  $\lVert\ell\rVert_\infty$ as the estimate $s^\cS_m\leq o(m)$.
\end{lemma}

\begin{proof}
  The result is trivial if $m=0$, so suppose $m > 0$. By symmetry, we may suppose without loss of generality that for each
  $i\in[k]$, $\sigma_i\colon[s^\cS_m]\to[m]$ is such that $\im(\sigma_i)=\{m-s^\cS_m+1,m-s^\cS_m+2,\ldots,m\}$. Pick $\rn{x}$ at
  random according to $\mu^m$ and for ease of notation, let
  \begin{align*}
    n & \df m - s^\cS_m, &
    \rn{H} & \df \rho^\cS_m\Bigl(\sigma^\#\bigl(\rn{x},F^*_m(\rn{x})\bigr),\eta\Bigr), &
    \rn{x'} & \df \iota^\#(\rn{x}),
  \end{align*}
  where $\iota=(\iota_i)_{i=1}^k$ is the $k$-tuple of inclusion maps $\iota_i\colon[n]\to[m]$ (so $\rn{x'}$ is the random
  sample corresponding to the first $n$ points of $\rn{x}$ in each of the $k$ coordinates, in their natural order).

  Note that
  \begin{align*}
    L_{\rn{x},F^*_m(\rn{x}),\ell}(\rn{H}) - L_{\rn{x'},F^*_n(\rn{x'}),\ell}(\rn{H})
    & =
    \begin{multlined}[t]
      \frac{1}{m^k}\cdot
      \sum_{\alpha\in[m]^k} \ell\Bigl(\alpha^*(\rn{x}), \rn{H}^*_m(\rn{x})_\alpha, F^*_m(\rn{x})\Bigr)
      \\
      -
      \frac{1}{n^k}\cdot
      \sum_{\alpha\in[n]^k} \ell\Bigl(\alpha^*(\rn{x}), \rn{H}^*_m(\rn{x})_\alpha, F^*_m(\rn{x})\Bigr)
    \end{multlined}
    \\
    & =
    \begin{multlined}[t]
      \left(\frac{1}{m^k} - \frac{1}{n^k}\right)\cdot
      \sum_{\alpha\in[n]^k} \ell\Bigl(\alpha^*(\rn{x}), \rn{H}^*_m(\rn{x})_\alpha, F^*_m(\rn{x})\Bigr)
      \\
      +
      \frac{1}{m^k}\cdot
      \sum_{\alpha\in[m]^k\setminus[n]^k} \ell\Bigl(\alpha^*(\rn{x}), \rn{H}^*_m(\rn{x})_\alpha, F^*_m(\rn{x})\Bigr),
    \end{multlined}
  \end{align*}
  and since $\ell$ is bounded and noting that the first term in the above is non-positive (as $n\leq m$) and the second term is
  non-negative, we conclude that
  \begin{gather}\label{eq:partconc:diff}
    \begin{aligned}
      \bigl\lvert L_{\rn{x},F^*_m(\rn{x}),\ell}(\rn{H}) - L_{\rn{x'},F^*_n(\rn{x'}),\ell}(\rn{H})\bigr\rvert
      & \leq
      \lVert\ell\rVert_\infty\cdot
      \max\left\{
      \left(\frac{1}{n^k} - \frac{1}{m^k}\right)\cdot n^k,
      \frac{1}{m^k}\cdot (m^k - n^k)
      \right\}
      \\
      & =
      \left(1 - \frac{n^k}{m^k}\right)\cdot\lVert\ell\rVert_\infty.
    \end{aligned}
  \end{gather}

  Since we are mostly concerned with the setting in which $s^\cS_m\leq o(m)$ (i.e., $n = (1-o(1))\cdot m$), the difference above
  is small so we will compare $L_{\mu,F,\ell}(\rn{H})$ with $L_{\rn{x'},F^*_n(\rn{x'}),\ell}$ first.

  Let $\cF_0$ be the $\sigma$-algebra generated by $((\rn{x}_i)_v)_{i\in[k],v\in[m]\setminus [n]}$ (i.e., generated by
  $\sigma^\#(\rn{x})$), so $\rn{H}$ is $\cF_0$-measurable. For each $t\in[k\cdot n]$, we let $\cF_i$ be the $\sigma$-algebra
  generated by $\cF_0$ and all $(\rn{x}_i)_v$ with $(i,v)\in[k]\times[n]$ satisfying
  \begin{gather*}
    (i-1)\cdot n + v-1 < t,
  \end{gather*}
  that is, from $\cF_{t-1}$ to $\cF_t$, we add the variable $(\rn{x}_{i_t})_{v_t}$, where
  \begin{align}\label{eq:partconc:itvt}
    i_t & \df \Floor{\frac{t-1}{n}} + 1, &
    v_t & \df \bigl((t-1)\bmod n\bigr) + 1.
  \end{align}

  We also let $\rn{Z}_t\df\EE[L_{\rn{x'},F^*_n(\rn{x'}),\ell}(\rn{H})\given\cF_t]$ and we note that $(\rn{Z}_t)_{t=0}^{k\cdot
    n}$ forms a (Doob) martingale with respect to $(\cF_t)_{t=0}^{k\cdot n}$:
  \begin{gather*}
    \EE[\rn{Z}_{t+1}\given\cF_t]
    =
    \EE\Bigl[
      \EE\bigl[L_{\rn{x'},F^*_n(\rn{x'}),\ell}(\rn{H})\given[\bigm]\cF_{t+1}\bigr]
      \given[\Bigm]
      \cF_t
      \Bigr]
    =
    \EE\bigl[L_{\rn{x'},F^*_n(\rn{x'}),\ell}(\rn{H})\given[\bigm]\cF_t\bigr]
    =
    \rn{Z}_t.
  \end{gather*}
  Note further that
  \begin{align*}
    \rn{Z}_0
    & =
    \EE\bigl[L_{\rn{x'},F^*_n(\rn{x'}),\ell}(\rn{H})\given[\bigm]\cF_0\bigr]
    =
    L_{\mu,F,\ell}(\rn{H}),
    \\
    \rn{Z}_{k\cdot n}
    & =
    \EE\bigl[L_{\rn{x'},F^*_n(\rn{x'}),\ell}(\rn{H})\given[\bigm]\cF_{k\cdot n}\bigr]
    =
    L_{\rn{x'},F^*_n(\rn{x'}),\ell}(\rn{H}),
  \end{align*}
  where the former is by linearity of conditional expectation (and the fact that $\rn{H}$ is $\cF_0$-measurable) and the latter
  is simply because the underlying random variable is $\cF_{k\cdot n}$-measurable.

  We will show that $\rn{Z}_0$ and $\rn{Z}_{k\cdot n}$ are close with high (conditional) probability using Azuma's Inequality.
  Toward that, we note that for each $t\in[k\cdot n]$ we have
  \begin{multline}\label{eq:partconc:martdiff}
    \lvert\rn{Z}_t - \rn{Z}_{t-1}\rvert
    =
    \Biggl\lvert
    \frac{1}{n^k}\cdot
    \sum_{\alpha\in[n]^k}
    \biggl(
    \EE\Bigl[
      \ell\bigl(\alpha^*(\rn{x'}), \rn{H}^*_n(\rn{x'}), \rn{F}^*_n(\rn{x'})\bigr)
      \given[\Bigm]
      \cF_t
      \Bigr]
    \\
    -
    \EE\Bigl[
      \ell\bigl(\alpha^*(\rn{x'}), \rn{H}^*_n(\rn{x'}), \rn{F}^*_n(\rn{x'})\bigr)
      \given[\Bigm]
      \cF_{t-1}
      \Bigr]
    \biggr)
    \Biggr\rvert.
  \end{multline}

  Let us now argue that most of the terms in the sum above are zero. Recall that what changes from $\cF_{t-1}$ to $\cF_t$ is the
  addition of the variable $(\rn{x}_{i_t})_{v_t}$, where $(i_t,v_t)$ are given by~\eqref{eq:partconc:itvt}. Consider a term
  $\alpha\in[n]^k$ such that $\alpha_{i_t}\neq v_t$, then
  \begin{gather*}
    \ell\bigl(\alpha^*(\rn{x'}), \rn{H}^*_n(\rn{x'}), \rn{F}^*_n(\rn{x'})\bigr)
  \end{gather*}
  is independent from $(\rn{x}_{i_t})_{v_t}$; in fact, since $((\rn{x}_i)_v)_{i\in[k], v\in[m]}$ is mutually independent, the
  above is conditionally independent from $(\rn{x}_{i_t})_{v_t}$ given $\cF_{t-1}$, so the above gives the same value when
  conditioned on $\cF_{t-1}$ as it does when conditioned on $\cF_t$. Thus, whenever $\alpha_{i_t}\neq v_t$, the corresponding
  term in~\eqref{eq:partconc:martdiff} is zero.

  As there are exactly $n^{k-1}$ many $\alpha\in[n]^k$ with $\alpha_{i_t} = v_t$, using triangle inequality and the fact that
  $\ell$ is bounded, we conclude that
  \begin{gather*}
    \lvert\rn{Z}_t - \rn{Z}_{t-1}\rvert
    \leq
    \frac{n^{k-1}}{n^k}\cdot\lVert\ell\rVert_\infty
    =
    \frac{\lVert\ell\rVert_\infty}{n}.
  \end{gather*}
  Recalling that
  \begin{gather*}
    \widetilde{\epsilon}
    \df
    \epsilon - \left(1 - \frac{n^k}{m^k}\right)\cdot\lVert\ell\rVert_\infty
    >
    0
  \end{gather*}
  (where the inequality follows from our assumption~\eqref{eq:partconc:condition}), by Azuma's Inequality (conditioned on
  $\cF_0$), we get
  \begin{align*}
    \PP_{\rn{x}\sim\mu^m}\biggl[
      L_{\mu,F,\ell}(\rn{H}) - L_{\rn{x'},F^*_n(\rn{x'}),\ell}(\rn{H})
      \geq
      \widetilde{\epsilon}
      \biggr]
    & =
    \PP[\rn{Z}_{k\cdot n} - \rn{Z}_0\leq - \widetilde{\epsilon}]
    \\
    & \leq
    \exp\left(-\frac{\widetilde{\epsilon}^2}{2\cdot\sum_{t=1}^{k\cdot n} (\lVert\ell\rVert_\infty/n)^2}\right)
    \\
    & =
    \exp\left(-\frac{\widetilde{\epsilon}^2\cdot n}{2\cdot k\cdot\lVert\ell\rVert_\infty^2}\right)
    \\
    & =
    \exp\left(-\frac{\widetilde{\epsilon}^2\cdot (m-s^\cS_m)}{2\cdot k\cdot\lVert\ell\rVert_\infty^2}\right).
  \end{align*}

  Putting this together with~\eqref{eq:partconc:diff}, we conclude that
  \begin{gather*}
    \PP_{\rn{x}\sim\mu^m}\bigl[
      L_{\mu,F,\ell}(\rn{H})
      -
      L_{\rn{x}, F^*_m(\rn{x}), \ell}(\rn{H})
      \geq
      \epsilon
      \bigr]
    \leq
    \exp\left(
    -\frac{\widetilde{\epsilon}\cdot(m-s^\cS_m)}{2\cdot k\cdot\lVert\ell\rVert_\infty^2}
    \right),    
  \end{gather*}
  as desired.

  \medskip

  Let us now estimate the above asymptotically when $s^\cS_m\leq o(m)$. First note that $1 - (m-s^\cS)^k/m^k = o(1)$ as
  $m\to\infty$ with $k$ fixed as $m - s^\cS_m = (1 - o(1))\cdot m$. Then we get
  \begin{align*}
    \exp\left(-\frac{\widetilde{\epsilon}^2\cdot (m-s^\cS_m)}{2\cdot k\cdot\lVert\ell\rVert_\infty^2}\right),
    & =
    \exp\left(
    -\frac{(\epsilon-\lVert\ell\rVert_\infty\cdot o(1))^2\cdot (1-o(1))\cdot m}{2\cdot k\cdot\lVert\ell\rVert_\infty^2}
    \right)
    \\
    & \leq
    \exp\left(
    -\frac{\epsilon^2}{2\cdot k\cdot\lVert\ell\rVert_\infty^2}\cdot m
    +\frac{\epsilon\cdot(\epsilon + \lVert\ell\rVert_\infty)}{\lVert\ell\rVert_\infty^2}\cdot o(m)
    \right),
  \end{align*}
  where the $o(m)$ error term is as $m\to\infty$ with $k$ fixed, is uniform in $\epsilon$, and has the same uniformity in
  $\lVert\ell\rVert_\infty$ as the estimate $s^\cS_m\leq o(m)$.
\end{proof}

The next theorem shows that every $k$-partite hypothesis class that admits a $k$-partite compression scheme $\cS$ of essentially
non-trivial quality must be $k$-PAC learnable. More specifically, a trivial $k$-partite compression scheme has $s^\cS_m=m$ and
$h^\cS_m=1$, which yields
\begin{gather*}
  \ln(h^\cS_m) + k\cdot s^\cS_m\cdot\ln(m) = k\cdot m\cdot\ln(m)
\end{gather*}
and our assumption is that we have a $k$-partite compression scheme in which the quantity above is $o(m)$ instead.

\begin{theorem}\label{thm:partSC->PAC}
  Let $k\in\NN_+$, let $\Omega=(\Omega_i)_{i=1}^k$ be a $k$-tuple of non-empty standard Borel spaces $\Omega_i=(X_i,\cB_i)$, let
  $\Lambda=(Y,\cB')$ be a non-empty standard Borel space with $\lvert Y\rvert\geq 2$, let $\ell\colon\prod_{i=1}^k X_i\times
  Y\times Y\to\RR_{\geq 0}$ be a bounded $k$-partite loss function and let $\cS$ be a $k$-partite compression scheme for a
  $k$-partite hypothesis class $\cH\subseteq\cF_k(\Omega,\Lambda)$ such that
  \begin{gather}\label{eq:partSC->PAC:condition}
    \ln(h^\cS_m) + k\cdot s^\cS_m\cdot\ln(m) \leq o(m)
  \end{gather}
  as $m\to\infty$ (which in particular follows from $b_\cS(m)\leq o(m/\ln(m))$ when $Y$ is finite), then $\cH$ is $k$-PAC
  learnable with respect to $\ell$. More specifically, letting $\cA$ be given by
  \begin{gather*}
    \cA(x,y)
    \df
    \rho^\cS_m\bigl(\kappa^\cS_m(x,y)\bigr)
    \qquad
    \left(m\in\NN, x\in\prod_{i=1}^k X_i^m, y\in Y^{([m])_k}\right)
  \end{gather*}
  gives a $k$-PAC learner for $\cH$ with respect to $\ell$ with learning guarantee
  \begin{multline}\label{eq:partSC->PAC:mPAC}
    m^{\PAC}_{\cH,\ell,\cA}(\epsilon,\delta)
    \df
    \min\Biggl\{m_0\in\NN_+,
    \;\Biggm\vert\;
    \forall m\geq m_0,
    \left(1 - \frac{(m-s^\cS_m)^k}{m^k}\right) < \epsilon
    \\
    \land
    (m)_{s^\cS_m}^k\cdot h^\cS_m\cdot
    \exp\left(
    -\frac{
      (\epsilon - (1-(m-s^\cS_m)^k/m^k)\cdot\lVert\ell\rVert_\infty)^2\cdot(m-s^\cS_m)
    }{
      2\cdot k\cdot\lVert\ell\rVert_\infty^2
    }
    \right)
    \leq
    \delta
    \Biggr\}.
  \end{multline}

  In particular, if
  \begin{gather}\label{eq:partSC->PAC:strongcondition}
    \ln(h^\cS_m) + k\cdot s^\cS_m\cdot\ln(m)
    \leq
    \bigl(1+o(1)\bigr)\cdot
    \sqrt{2\cdot k\cdot m}\cdot\ln(m),
  \end{gather}
  as $m\to\infty$, then
  \begin{gather*}
    m^{\PAC}_{\cH,\ell,\cA}(\epsilon,\delta)
    \leq
    \bigl(1 + o(1)\bigr)\cdot
    \frac{2\cdot k\cdot\lVert\ell\rVert_\infty^2}{\epsilon^2}\cdot
    \max\left\{1, \ln\left(\frac{1}{\delta}\right)\right\},
  \end{gather*}
  where the $o(1)$ error term is as $\epsilon\to 0$ with $k$ fixed, is uniform in $\delta$, and has the same uniformity in
  $\lVert\ell\rVert_\infty$ as the estimate $o(1)$ in~\eqref{eq:partSC->PAC:strongcondition} provided $\lVert\ell\rVert_\infty$
  is bounded away from $0$.
\end{theorem}

\begin{proof}
  We start by showing that $b_\cS(m)\leq o(m/\ln(m))$ (when $Y$ is finite) implies the
  condition~\eqref{eq:partSC->PAC:condition} (i.e., $\ln(h^\cS_m) + k\cdot s^\cS_m\cdot\ln(m) \leq o(m)$). Indeed, for $m\geq
  e$, we have
  \begin{gather*}
    \ln(h^\cS_m) + k\cdot s^\cS_m\cdot\ln(m)
    \leq
    \ln(m)\cdot\bigl(\log_2(h^\cS_m) + k\cdot s^\cS_m\cdot\log_2\lvert Y\rvert\bigr)
    \leq
    \ln(m)\cdot b_\cS(m)
    \leq
    o(m).
  \end{gather*}

  Let us now argue that the minimum in~\eqref{eq:partSC->PAC:mPAC} indeed exists. For this, it suffices to argue that there
  exists $m_0\in\NN$ satisfying both conditions in the definition of the minimum. But indeed, since we are assuming
  that~\eqref{eq:partSC->PAC:condition} (i.e., $\ln(h^\cS_m) + k\cdot s^\cS_m\cdot\ln(m) \leq o(m)$) holds, it follows that both
  \begin{gather*}
    \left(1 - \frac{(m-s^\cS_m)^k}{m^k}\right)\cdot\lVert\ell\rVert_\infty,
    \\
    (m)_{s^\cS_m}^k\cdot h^\cS_m\cdot
      \exp\left(
      -\frac{
        (\epsilon - (1-(m-s^\cS_m)^k/m^k)\cdot\lVert\ell\rVert_\infty)^2\cdot(m-s^\cS_m)
      }{
        2\cdot k\cdot\lVert\ell\rVert_\infty^2
      }
      \right)
  \end{gather*}
  converge to $0$ as $m\to\infty$ (with all other parameters fixed). So there must exist $m_0\in\NN$ large enough so that for
  every $m\geq m_0$, the first expression above is less than $\epsilon$ and the second expression above is at most $\delta$.

  \medskip

  We now prove the result. Let $\mu\in\Pr(\Omega)$ and $F\in\cF_k(\Omega,\Lambda)$ be a $k$-partite hypothesis that is
  realizable with respect to $\mu$ and $\ell$. Given $\epsilon,\delta\in(0,1)$, let $m\geq
  m^{\PAC}_{\cH,\ell,\cA}(\epsilon,\delta)$ be an integer and let $\widetilde{\epsilon}$ be defined in terms of $m$, $s^\cS_m$,
  $k$ and $\ell$ as in Lemma~\ref{lem:partconc} and note that since $m\geq m^{\PAC}_{\cH,\ell,\cA}(\epsilon,\delta)$, the first
  condition in the definition of the latter implies that the condition~\eqref{eq:partconc:condition} of Lemma~\ref{lem:partconc}
  is satisfied.

  We now pick $\rn{x}$ at random according to $\mu^m$ and for each $k$-tuple $\sigma=(\sigma_i)_{i=1}^k$ of injections
  $[s^\cS_m]\to[m]$ and each $\eta\in[h^\cS_m]$, we let $E_{\sigma,\eta}(\rn{x})$ be the event
  \begin{gather*}
    L_{\mu,F,\ell}\biggl(\rho^\cS_m\Bigl(\sigma^\#\bigl(\rn{x}, F^*_m(\rn{x})\bigr),\eta\Bigr)\biggr)
    -
    L_{\rn{x}, F^*_m(\rn{x}), \ell}\biggl(\rho^\cS_m\Bigl(\sigma^\#\bigl(\rn{x}, F^*_m(\rn{x})\bigr),\eta\Bigr)\biggr)
    \geq
    \epsilon
  \end{gather*}
  so that Lemma~\ref{lem:partconc} says that
  \begin{gather*}
    \PP_{\rn{x}\sim\mu^m}\bigl[E_{\sigma,\eta}(\rn{x})\bigr]
    \leq
    \exp\left(-\frac{\widetilde{\epsilon}^2\cdot(m-s^\cS_m)}{2\cdot k\cdot\lVert\ell\rVert_\infty^2}\right).
  \end{gather*}

  If we let $E(\rn{x})$ be the disjunction of the events $E_{\sigma,\eta}(\rn{x})$, since there are $(m)_{s^\cS_m}^k$ many such
  $\sigma$ and $h^\cS_m$ many such $\eta$, by the union bound, we get
  \begin{gather*}
    \PP_{\rn{x}\sim\mu^m}\bigl[E(\rn{x})\bigr]
    \leq
    (m)_{s^\cS_m}^k\cdot h^\cS_m\cdot
    \exp\left(-\frac{\widetilde{\epsilon}^2\cdot(m-s^\cS_m)}{2\cdot k\cdot\lVert\ell\rVert_\infty^2}\right).
  \end{gather*}

  We now note that we have the compression scheme guarantee that
  \begin{gather*}
    L_{\rn{x},F^*_m(\rn{x}),\ell}\biggl(\rho^\cS_m\Bigl(\kappa^\cS_m\bigl(\rn{x}, F^*_m(\rn{x})\bigr)\Bigr)\biggr)
    =
    0
  \end{gather*}
  whenever $(\rn{x},F^*_m(\rn{x}))$ is sample realizable in $\cH$ with respect to $\ell$, which happens with probability $1$ as
  $F$ is realizable in $\cH$ with respect to $\ell$.

  The above along with the fact that the (random) injection $\sigma^\cS_m(\rn{x}, F^*_m(\rn{x}))$ and the (random) header
  $\eta^\cS_m(\rn{x},F^*_m(\rn{x}))$ are elements of $([m])_{s^\cS}^k$ and $[h^\cS_m]$, respectively, implies that
  \begin{align*}
    \MoveEqLeft
    \PP_{\rn{x}\sim\mu^m}\biggl[
      L_{\mu,F,\ell}\Bigl(\cA\bigl(\rn{x}, F^*_m(\rn{x})\bigr)\Bigr) > \epsilon
      \biggr]
    \\
    & =
    \PP_{\rn{x}\sim\mu^m}\Biggl[
      L_{\mu,F,\ell}\biggl(\rho^\cS_m\Bigl(\kappa^\cS_m\bigl(\rn{x}, F^*_m(\rn{x})\bigr)\Bigr)\biggr) > \epsilon
      \Biggr]
    \\
    & \leq
    \PP_{\rn{x}\sim\mu^m}\bigl[E(\rn{x})\bigr]
    \\
    & \leq
    (m)_{s^\cS_m}^k\cdot h^\cS_m\cdot
    \exp\left(-\frac{\widetilde{\epsilon}^2\cdot(m-s^\cS_m)}{2\cdot k\cdot\lVert\ell\rVert_\infty^2}\right).
    \\
    & =
    (m)_{s^\cS_m}^k\cdot h^\cS_m\cdot
    \exp\left(
    -\frac{
      (\epsilon - (1-(m-s^\cS_m)^k/m^k)\cdot\lVert\ell\rVert_\infty)^2\cdot(m-s^\cS_m)
    }{
      2\cdot k\cdot\lVert\ell\rVert_\infty^2
    }
    \right)
    \\
    & \leq
    \delta,
  \end{align*}
  where the last inequality follows from $m\geq m^{\PAC}_{\cH,\ell,\cA}(\epsilon,\delta)$ and using the second condition in the
  definition of the latter. Thus $\cA$ is a $k$-PAC learner with respect to $\ell$.

  \medskip

  It remains to make the asymptotic estimate of $m^{\PAC}_{\cH,\ell,\cA}(\epsilon,\delta)$ when the stronger
  assumption~\eqref{eq:partSC->PAC:strongcondition} (i.e., $ \ln(h^\cS_m) + k\cdot s^\cS_m\cdot\ln(m) \leq (1+o(1))\cdot
  \sqrt{2\cdot k\cdot m}\cdot\ln(m)$) holds. (It is clear that~\eqref{eq:partSC->PAC:strongcondition} implies the
  condition~\eqref{eq:partSC->PAC:condition} that $\ln(h^\cS_m) + k\cdot s^\cS_m\cdot\ln(m) \leq o(m)$.)

  Let $m_1(\epsilon)\in\NN_+$ be the least integer such that for every $m\geq m_1(\epsilon)$, we have
  \begin{gather}\label{eq:partSC->PAC:m1def}
    \left(1 - \frac{(m-s^\cS)^k}{m^k}\right)\cdot\lVert\ell\rVert_\infty < \epsilon
  \end{gather}
  and let $m_2(\epsilon,\delta)$ be the least integer such that for every $m\geq m_2(\epsilon,\delta)$, we have
  \begin{gather}\label{eq:partSC->PAC:m2def}
    (m)_{s^\cS_m}^k\cdot h^\cS_m\cdot
    \exp\left(
    -\frac{
      (\epsilon - (1-(m-s^\cS_m)^k/m^k)\cdot\lVert\ell\rVert_\infty)^2\cdot(m-s^\cS_m)
    }{
      2\cdot k\cdot\lVert\ell\rVert_\infty^2
    }
    \right)
    \leq
    \delta,
  \end{gather}
  Since $m^{\PAC}_{\cH,\ell,\cA}(\epsilon,\delta)=\max\{m_1(\epsilon),m_2(\epsilon,\delta)\}$, it suffices to give asymptotic
  estimates for $m_1(\epsilon)$ and $m_2(\epsilon,\delta)$.

  To estimate $m_1(\epsilon)$, note that if $\lVert\ell\rVert_\infty < \epsilon$, then $m_1(\epsilon) = 1$, so suppose that
  $\lVert\ell\rVert_\infty\geq\epsilon$ and note that
  \begin{gather}\label{eq:partSC->PAC:towardm1}
    \begin{aligned}
      0
      & \leq
      1 - \frac{(m-s^\cS_m)^k}{m^k}
      =
      1 - \left(1 - \frac{s^\cS_m}{m}\right)^k
      \\
      & \leq
      1 - \left(1 - \bigl(1 + o(1)\bigr)\cdot\sqrt{\frac{2}{k\cdot m}}\right)^k
      =
      \bigl(1 + o(1)\bigr)\cdot
      \sqrt{\frac{2\cdot k}{m}}
    \end{aligned}
  \end{gather}
  as $m\to\infty$ with $k$ fixed.

  Plugging this asymptotic estimate in the definition~\eqref{eq:partSC->PAC:m1def} of $m_1(\epsilon)$, we conclude that
  \begin{gather}\label{eq:partSC->PAC:m1estimate}
    m_1(\epsilon)
    \leq
    \bigl(1 + o(1)\bigr)
    \frac{2\cdot k\cdot\lVert\ell\rVert_\infty^2}{\epsilon^2}
  \end{gather}
  where the $o(1)$ error term is as $\epsilon\to 0$ with $k$ fixed, is uniform in $\delta$ (as $m_1(\epsilon)$ does not depend
  on $\delta$ at all), and has the same uniformity in $\lVert\ell\rVert_\infty$ as the estimate $o(1)$
  in~\eqref{eq:partSC->PAC:strongcondition} provided $\lVert\ell\rVert_\infty$ is bounded away from $0$ (this is to cover the
  case $\lVert\ell\rVert_\infty < \epsilon$ when $m_1(\epsilon) = 1$).

  To estimate $m_2(\epsilon,\delta)$, we use~\eqref{eq:partSC->PAC:towardm1} along with the weaker
  condition~\eqref{eq:partSC->PAC:condition} (i.e., $\ln(h^\cS_m) + k\cdot s^\cS_m\cdot\ln(m) \leq o(m)$) to get
  \begin{align*}
    \MoveEqLeft
    (m)_{s^\cS_m}^k\cdot h^\cS_m\cdot
    \exp\left(
    -\frac{
      (\epsilon - (1-(m-s^\cS_m)^k/m^k)\cdot\lVert\ell\rVert_\infty)^2\cdot(m-s^\cS_m)
    }{
      2\cdot k\cdot\lVert\ell\rVert_\infty^2
    }
    \right)
    \\
    & \leq
    \exp\left(
    -\frac{
      (\epsilon - o(1)\cdot\lVert\ell\rVert_\infty)^2\cdot(1+o(1))\cdot m
    }{
      2\cdot k\cdot\lVert\ell\rVert_\infty^2
    }
    + o(m)
    \right)
    \\
    & \leq
    \exp\left(
    -\left(\frac{\epsilon^2}{2\cdot k\cdot\lVert\ell\rVert_\infty^2} - o(1)\right)\cdot m
    \right),
  \end{align*}
  where the error terms are as $m\to\infty$ with $k$ fixed, are uniform in $\epsilon$, and have the same uniformity in
  $\lVert\ell\rVert_\infty$ as the estimate $o(1)$ in~\eqref{eq:partSC->PAC:strongcondition}.

  Plugging this asymptotic estimate in the definition~\eqref{eq:partSC->PAC:m2def} of $m_2(\epsilon,\delta)$, we conclude that
  \begin{gather*}
    m_2(\epsilon,\delta)
    \leq
    \left(\frac{2\cdot k\cdot\lVert\ell\rVert^2}{\epsilon^2} + o(1)\right)\cdot\ln\left(\frac{1}{\delta}\right),
  \end{gather*}
  where the $o(1)$ error term is as $\epsilon\to 0$, is uniform in $\delta$, and has the same uniformity in
  $\lVert\ell\rVert_\infty$ as the estimate $o(1)$ in~\eqref{eq:nonpartSC->PAC:strongcondition}.

  Putting the above together with our estimate~\eqref{eq:partSC->PAC:m1estimate} for $m_1(\epsilon)$, we get
  \begin{gather*}
    m^{\PAC}_{\cH,\ell,\cA}(\epsilon,\delta)
    \leq
    \bigl(1 + o(1)\bigr)\cdot
    \frac{2\cdot k\cdot\lVert\ell\rVert_\infty^2}{\epsilon^2}\cdot\max\left\{1, \ln\left(\frac{1}{\delta}\right)\right\},
  \end{gather*}
  as desired since $m^{\PAC}_{\cH,\ell,\cA} = \max\{m_1(\epsilon),m_2(\epsilon,\delta)\}$.
\end{proof}

\section{Non-partite setting}

In this section, we cover the non-partite case, that is, the case in which our hypotheses classes are collections of functions
of the form $F\colon X^k\to Y$, $k$-PAC learning notions are with respect to power measures $\mu^k$ of the same probability
measure on $X$ and samples are thought of as having $m$ points $x_1,\ldots,x_m$ from $X$ as well as labels
$F(x_{\alpha_1},\ldots,x_{\alpha_k}^k)\in Y$ for each injective $k$-tuple $\alpha\in([m])_k$.

Let us point out some of the peculiarities of the non-partite case:
\begin{itemize}
\item Labels on $k$-ary samples only take into account injective tuples, that is, once we are given $m$ points
  $x_1,\ldots,x_m\in X$, to form a labeled sample, we only collect labels from a hypothesis $F\colon X^k\to Y$ of the form
  $F(x_{\alpha_1},\ldots,x_{\alpha_k}^k)\in Y$ with $\alpha\colon[k]\to[m]$ injective. The reason why we do not collect labels
  of non-injective tuples is technical and comes from the agnostic setup via exchangeability theory. Nevertheless, we point out
  that even if we included non-injective tuples, since these account for a negligible proportion of all tuples (for $m$ large),
  this would only change bounds on lower order terms by discounting non-injective tuples from calculations.
\item Loss functions $\ell$ collect labels from all orientations of a $k$-set, that is, given a $k$-tuple $(x_1,\ldots,x_k)\in
  X^k$ and hypotheses $F,H\colon X^k\to Y$, we collect the labels of the hypotheses on all possible orientations of the
  $k$-tuple $(x_1,\ldots,x_k)$ into an element of $Y^{S_k}$ (recall that $S_k=([k])_k$ is the symmetric group on $[k]$) before
  feeding them into the loss function. This accounts for the fact that $X$ does not come with an inherent ordering and allows
  loss functions to take into account other orientations of the $k$-tuples before assigning a penalty.
\item When computing the empirical loss on a sample of $m$ points $x_1,\ldots,x_m$, the $k$-tuples of it do not come with an
  inherent order, so there must be an order choice that picks one particular orientation for each $k$-subset of $[m]$ as the
  ``standard direction'' in which the loss function is going to be computed.
\end{itemize}

We again follow a simplified notation of~\cite{CM24}.

\begin{definition}[Non-partite setting]\label{def:nonpart}
  Let $k\in\NN_+$, let $\Omega=(X,\cB)$ and $\Lambda=(Y,\cB')$ be non-empty standard Borel spaces and let $m\in\NN$.
  \begin{enumdef}
  \item\label{def:nonpart:alpha*} A \emph{$k$-ary unlabeled sample of size $m$} with respect to $\Omega$ is an
    element of $X^m$. A \emph{$k$-ary labeled sample} of size $m$ with respect to $\Omega$ and $\Lambda$ is an element of
    $X^m\times Y^{([m])_k}$; we think of the coordinate $y_\alpha$ indexed by some injective $k$-tuple $\alpha\in([m])_k$ as the
    label on $x_{\alpha_1},\ldots,x_{\alpha_k}$.

    More generally, every injection $\alpha\colon[n]\to[m]$ induces naturally a \emph{subsample} $\alpha^*(x)$ of size $n$ by
    picking the points and labels corresponding to the coordinates indexed by $\alpha$ (in the order of $\alpha$). This is
    formalized as follows: we contra-variantly define the maps
    \begin{align*}
      \alpha^*\colon X^m\to X^n, & &
      \alpha^*\colon Y^{([m])_k}\to Y^{([n])_k}
    \end{align*}
    that select the points and labels according to
    \begin{align*}
      \alpha^*(x)_i & \df x_{\alpha(i)}
      \qquad \bigl(x\in X^m, i\in[n]\bigr),
      \\
      \alpha^*(y)_\beta & \df y_{\alpha\comp\beta}
      \qquad \bigl(y\in Y^{([m])_k}, \beta\in([m])_k\bigr).
    \end{align*}
    With slight abuse of notation, we also contra-variantly define the map
    \begin{gather*}
      \alpha^*\colon X^m\times Y^{([m])_k}\to X^n\times Y^{([n])_k}
    \end{gather*}
    by letting it act as the previous two maps in the corresponding coordinates.   
  \item\label{def:nonpart:F*} A \emph{$k$-ary hypothesis} from $\Omega$ to $\Lambda$ is a measurable function $F\colon X_i\to
    Y$. The set of $k$-ary hypothesis from $\Omega$ to $\Lambda$ is denoted $\cF_k(\Omega,\Lambda)$.

    Given a $k$-ary hypothesis $F\in\cF_k(\Omega,\Lambda)$ and $m\in\NN$, we define the function $F^*_m\colon X^m\to
    Y^{([m])_k}$ that when given an unlabeled sample of size $m$, provides the labels of $F$ on the injective $k$-tuples of the
    sample; in a formula:
    \begin{gather*}
      F^*_m(x)_\alpha \df F\bigl(\alpha^*(x)\bigr)
      \qquad \bigl(x\in X^m, \alpha\in([m])_k\bigr).
    \end{gather*}

    Similarly to the partite case, this definition is equivariant in the sense that the following diagram is commutative:
    \begin{equation*}
      \begin{tikzcd}
        X^m
        \arrow[r, "F^*_m"]
        \arrow[d, "\alpha^*"']
        &
        Y^{([m])_k}
        \arrow[d, "\alpha^*"]
        \\
        X^n
        \arrow[r, "F^*_n"]
        &
        Y^{([n])_k}
      \end{tikzcd}
    \end{equation*}
  \item A \emph{$k$-ary hypothesis class} is a set $\cH\subseteq\cF_k(\Omega,\Lambda)$ of $k$-ary hypotheses that is
    further equipped with a $\sigma$-algebra such that:
    \begin{itemize}
    \item the evaluation map $\ev\colon\cH\times X^k\to Y$ given by $\ev(H,x)\df H(x)$ is measurable;
    \item for every $H\in\cH$, the singleton $\{H\}$ is measurable;
    \item for every Borel space $\Upsilon$ and every measurable set $A\subseteq\cH\times\Upsilon$, the projection of $A$ onto
      $\Upsilon$, i.e., the set
      \begin{equation*}
        \{\upsilon\in\Upsilon \mid \exists H\in\cH, (H,\upsilon)\in A\}
      \end{equation*}
      is universally measurable\footnote{Similarly to the partite case, the reader unfamiliar with these technical measurability
      conditions can simply assume that $\cH$ is equipped with a $\sigma$-algebra that makes it a standard Borel space as it
      will imply this condition; or they can simply interpret this as ``all probabilities computed must make sense''.}.
    \end{itemize}
  \item A \emph{$k$-ary loss function} over $\Lambda$ is a measurable function $\ell\colon X^k\times Y^{S_k}\times
    Y^{S_k}\to\RR_{\geq 0}$ that when given a $k$-tuple $x$ and two labels $y$ and $y'$ of all possible orientations of the
    $k$-tuple $x$, assigns a value $\ell(x,y,y')$ that is the penalty of guessing $y$ on the tuple $x$ when the correct label
    was $y'$.

    We define
    \begin{gather*}
      \lVert\ell\rVert_\infty \df \sup_{\substack{x\in X^k\\y,y'\in Y^{S_k}}} \ell(x,y,y')
    \end{gather*}
    and we say that $\ell$ is \emph{bounded} if $\lVert\ell\rVert_\infty < \infty$.

    Given further $\mu\in\Pr(\Omega)$ and $k$-ary hypotheses $F$ and $H$, the \emph{total loss} of $H$ with respect to $\mu$,
    $F$ and $\ell$ is
    \begin{gather*}
      L_{\mu,F,\ell}(H) \df \EE_{\rn{x}\sim\mu^k}\Bigl[\ell\bigl(\rn{x}, H^*_k(\rn{x}), F^*_k(\rn{x})\bigr)\Bigr].
    \end{gather*}

    We say that $F$ is \emph{realizable} in a $k$-ary hypothesis class $\cH\subseteq\cF_k(\Omega,\Lambda)$ with respect to
    $\mu$ and $\ell$ if $\inf_{H\in\cH} L_{\mu,F,\ell}(H) = 0$.
  \item For a $k$-ary hypothesis class $\cH'\subseteq\cF_k(\Omega,\Lambda)$, a \emph{($k$-ary) learning algorithm} with
    outputs in $\cH'$ is a measurable function
    \begin{gather*}
      \cA\colon\bigcup_{m\in\NN} (X^m\times Y^{([m])_k}) \to \cH'.
    \end{gather*}

    We say that a $k$-ary hypothesis class $\cH\subseteq\cF_k(\Omega,\Lambda)$ is \emph{improperly $k$-PAC learnable} with
    respect to a $k$-ary loss function $\ell\colon X^k\times Y^{S_k}\times Y^{S_k}\to\RR_{\geq 0}$ if there exists a learning
    algorithm $\cA$ with outputs in some $k$-ary hypothesis class $\cH'$ and a function
    $m^{\PAC}_{\cH,\ell,\cA}\colon(0,1)^2\to\RR_{\geq 0}$ such that for every $\epsilon,\delta\in(0,1)$, every
    $\mu\in\Pr(\Omega)$, every $F\in\cF_k(\Omega,\Lambda)$ that is realizable in $\cH$ with respect to $\mu$ and $\ell$ and
    every integer $m\geq m^{\PAC}_{\cH,\ell,\cA}(\epsilon,\delta)$, we have
    \begin{gather*}
      \PP_{\rn{x}\sim\mu^m}\biggl[
        L_{\mu,F,\ell}\Bigl(\cA\bigl(\rn{x}, F^*_m(\rn{x})\bigr)\Bigr)\leq\epsilon
        \biggr]
      \geq
      1 - \delta.
    \end{gather*}
    In plain English: when given a random sample of size $m$ drawn from $\mu^m$ labeled according to some realizable $F$, the
    algorithm $\cA$ outputs some hypothesis $H$ with total loss at most $\epsilon$ with probability at least $1-\delta$ over the
    randomness of the sample.

    The function $m^{\PAC}_{\cH,\ell,\cA}$ is called the \emph{learning guarantee} of $\cA$ and $\cA$ is called an \emph{improper
    $k$-PAC learner} for $\cH$ with respect to $\ell$.

    The notions of \emph{proper $k$-PAC learnability} and \emph{proper $k$-PAC learner} are defined analogously but requiring
    further that $\cH'=\cH$.
  \item A \emph{($k$-ary) order choice} for $[m]$ is a sequence $\alpha=(\alpha_U)_{U\in\binom{V}{k}}$ of injections of the form
    $\alpha_U\colon[k]\to[m]$ with $\im(\alpha_U) = U$.

    Any such order choice $\alpha$ naturally induces a map $b_\alpha\colon Y^{([m])_k}\to (Y^{S_k})^{\binom{[m]}{k}}$ that
    bundles up all labels corresponding to orientations of a $k$-set $U$ into the coordinate indexed by $U$, where the standard
    orientation is considered to be in the direction of $\alpha_U$; formally, the map $b_\alpha$ is given by
    \begin{gather*}
      \bigl(b_\alpha(y)_U\bigr)_\pi \df y_{\alpha_U\comp\pi}
      \qquad
      \left(y\in Y^{([m])_k}, U\in\binom{[m]}{k}, \pi\in S_k\right).
    \end{gather*}
    
    Given further $k$-ary labeled sample $(x,y)\in(\prod_{i=1}^k X_i^{m_i})\times Y^{\prod_{i=1}^k [m_i]}$ of sizes $m$, a $k$-ary
    loss function $\ell\colon X^k\times Y^{S_k}\times Y^{S_k}\to\RR_{\geq 0}$ and a hypothesis $H\in\cF_k(\Omega,\Lambda)$, the
    \emph{empirical loss} (or \emph{empirical risk}) of $H$ with respect to $(x,y)$, $\ell$ and $\alpha$ is
    \begin{gather*}
      L^\alpha_{x,y,\ell}(H) \df
      \begin{dcases*}
        \frac{1}{\binom{m}{k}}\cdot
        \sum_{U\in\binom{[m]}{k}} \ell\Bigl(\alpha_U^*(x), b_\alpha\bigl(H^*_m(x)\bigr)_U, b_\alpha(y)_U\Bigr),
        & if $m\geq k$,
        \\
        0, & if $m < k$.
      \end{dcases*}
    \end{gather*}

    We say that $(x,y)$ is \emph{sample realizable} in a $k$-ary hypothesis class $\cH\subseteq\cF_k(\Omega,\Lambda)$ with
    respect to $\ell$ and $\alpha$ if $\inf_{H\in\cH} L^\alpha_{x,y,\ell}(H)$. It is straightforward to check that if
    $\mu\in\Pr(\Omega)$, $F\in\cF_k(\Omega,\Lambda)$ is realizable in $\cH$ with respect to $\mu$ and $\ell$ and
    $\rn{x}\sim\mu^m$, then with probability $1$, we have that $(\rn{x}, F^*_m(\rn{x}))$ is sample realizable in $\cH$ with
    respect to $\ell$ and $\alpha$.
  \end{enumdef}
\end{definition}

The next definition is that of a high-arity selection/compression scheme in the non-partite case. Note that since the definition
of a compression schemes rely on empirical loss, in the non-partite case, we will also need to specify an order choice.

\begin{definition}[Non-partite setting]\label{def:nonpartcomp}
  Let $k\in\NN_+$ and let $\Omega=(X,\cB)$ and $\Lambda=(Y,\cB')$ be non-empty standard Borel spaces.
  \begin{enumdef}
  \item A \emph{$k$-ary selection scheme} with outputs in a $k$-ary hypothesis class
    $\cH'\subseteq\cF_k(\Omega,\Lambda)$ is a tuple $\cS=(\sigma,\eta,\rho)$ such that:
    \begin{itemize}
    \item $\sigma=(\sigma_m)_{m\in\NN}$ is a sequence of maps, called \emph{selection maps}, such that each $\sigma_m$ when
      given a sample of size $m$, outputs a tuple of some length $s_m$ that contains the indices of the subsample selected.
      Formally, each $\sigma_m$ is a measurable function
      \begin{gather*}
        \sigma_m\colon X^m\times Y^{([m])_k}\to ([m])_{s_m}
      \end{gather*}
      where $s_m\in\NN$ ($i\in[k]$) is some integer, which we call the \emph{selection size} (note that we must have $s_m\leq m$
      as the domain of $\sigma_m$ is never empty).
    \item $\eta=(\eta_m)_{m\in\NN}$ is a sequence of maps, called the \emph{header maps}, that when given a sample of size $m$,
      outputs some extra information, encoded as an element of $[h_m]$ for some fixed $h_m\in\NN$, that will aid in the
      reconstruction of the original sample. Formally, each $\eta_m$ is a measurable function
      \begin{gather*}
        \eta_m\colon X^m\times Y^{([m])_k}\to [h_m],
      \end{gather*}
      where $h_m\in\NN_+$ is called the \emph{header size} of $\eta_m$.
    \item $\rho=(\rho_m)_{m\in\NN}$ is a sequence of maps, called the \emph{reconstruction maps}, that when given a labeled
      sample of size $s_m$ and header of size $h_m$ (intended to have come from $\sigma_m$ and $\eta_m$),
      reconstruct a labeled sample of size $m$. Formally, each $\rho_m$ is a measurable function
      \begin{gather*}
        \rho_m\colon X^{s_m}\times Y^{([s_m])_k}\times[h_m]\to\cH'.
      \end{gather*}
    \end{itemize}

    The compression procedure of $\cS$ is formalized by defining a sequence $\kappa=(\kappa_m)_{m\in\NN}$ of \emph{compression
    maps} of the form
    \begin{gather*}
      \kappa_m\colon X^m\times Y^{([m])_k}
      \to
      X^{s_m}\times Y^{([s_m])_k}\times[h_m]
    \end{gather*}
    by
    \begin{gather*}
      \kappa_m(x,y)
      \df
      \bigl(\sigma_m(x,y)^*(x,y), \eta_m(x,y)\bigr)
      \qquad
      \bigl(x\in X^m, y\in Y^{([m])_k}\bigr).
    \end{gather*}
    (Recall from Definition~\ref{def:nonpart:alpha*} that since $\sigma_m(x,y)\in([m])_{s_m}$ is a injection, it contra-variantly
    induces a map that selects a subsample $\sigma_m(x,y)^*(x,y)$ of size $s_m$ from the sample $(x,y)$ of
    size $m$ and the map $\kappa_m$ then takes this subsample and appends to it the header $\eta_m(x,y)$.)

    We will use the notation
    \begin{align*}
      \sigma^\cS & \df \sigma = (\sigma_m)_{m\in\NN}, &
      \eta^\cS & \df \eta = (\eta_m)_{m\in\NN}, &
      \rho^\cS & \df \rho = (\rho_m)_{m\in\NN},
      \\
      s^\cS_m & \df s_m, &
      h^\cS_m & \df h_m, &
      \kappa^\cS & \df \kappa = (\kappa_m)_{m\in\NN}.
    \end{align*}

    The \emph{compression size} and the \emph{compression bitlength} functions of $\cS$ are the functions
    $c_\cS,b_\cS\colon\NN\to\RR_{\geq 0}$ given by
    \begin{align*}
      c_\cS(m) & \df h_m\cdot\lvert Y\rvert^{(s_m)_k}, &
      b_\cS(m) & \df \log_2\bigl(c_\cS(m)\bigr).
    \end{align*}
  \item A $k$-ary selection scheme $\cS$ with outputs in a $k$-ary hypothesis class $\cH'\subseteq\cF_k(\Omega,\Lambda)$
    is called a \emph{sample compression scheme} for a $k$-ary hypothesis class $\cH\subseteq\cF_k(\Omega,\Lambda)$ with
    respect to a $k$-ary loss function $\ell\colon X^k\times Y^{S_k}\times Y^{S_k}\to\RR_{\geq 0}$ if for every $m\in\NN$
    and every labeled sample $(x,y)\in X^m\times Y^{([m])_k}$ of size $m$ that is sample realizable with respect
    to $\ell$ and every order choice $\alpha$ for $[m]$, we have
    \begin{gather*}
      L^\alpha_{x,y,\ell}\Bigl(\rho^\cS_m\bigl(\kappa^\cS_m(x,y)\bigr)\Bigr) = 0,
    \end{gather*}
    or in plain English: if we compress $(x,y)$ via $\kappa^\cS$ then reconstruct via $\rho^\cS$, then the result has zero
    empirical loss.
  \end{enumdef}
\end{definition}

Our goal is to prove the non-partite analogue of Theorem~\ref{thm:partSC->PAC}. We start with a concentration lemma based on
Azuma's Inequality that is the non-partite analogue of Lemma~\ref{lem:partconc}; the main differences of this lemma to its
partite counterpart is that we need to keep track of the order choice $\alpha$ and the bounds are slightly different due to
normalization by $\binom{m}{k}$ as opposed to $m^k$ and the fact that the martingale has $m$ steps instead of $k\cdot m$ steps.

\begin{lemma}\label{lem:nonpartconc}
  Let $\Omega=(X,\cB)$ and $\Lambda=(Y,\cB')$ be non-empty standard Borel spaces, let $k\in\NN_+$, let $\cS$ be a $k$-ary
  selection scheme and let $\ell\colon X^k\times Y^{S_k}\times Y^{S_k}\to\RR_{\geq 0}$ be a bounded $k$-ary loss
  function.

  Then for every $\epsilon > 0$, every probability measure $\mu\in\Pr(\Omega)$, every integer $m\in\NN$ such that
  \begin{gather}\label{eq:nonpartconc:condition}
    \left(1 - \frac{(m-s^\cS_m)_k}{(m)_k}\right)\cdot\lVert\ell\rVert_\infty < \epsilon,
  \end{gather}
  every order choice $\alpha$ for $[m]$, every $\eta\in[h^\cS_m]$, every injection $\sigma\colon[s^\cS_m]\to[m]$ and every
  $F\in\cF_k(\Omega,\Lambda)$, we have
  \begin{multline*}
    \PP_{\rn{x}\sim\mu^m}\Biggl[
      L_{\mu,F,\ell}\biggl(\rho^\cS\Bigl(\sigma^*\bigl(\rn{x},F^*_m(\rn{x})\bigr), \eta\Bigr)\biggr)
      -
      L^\alpha_{\rn{x}, F^*_m(\rn{x}), \ell}\biggl(\rho^\cS\Bigl(\sigma^*\bigl(\rn{x},F^*_m(\rn{x})\bigr), \eta\Bigr)\biggr)
      \geq
      \epsilon
      \Biggr]
    \\
    \leq
    \exp\left(-\frac{\widetilde{\epsilon}^2\cdot(m-s^\cS_m)}{2\cdot k^2\cdot\lVert\ell\rVert_\infty^2}\right),
  \end{multline*}
  where
  \begin{gather*}
    \widetilde{\epsilon}
    \df
    \epsilon - \left(1 - \frac{(m-s^\cS_m)_k}{(m)_k}\right)\cdot\lVert\ell\rVert_\infty
    >
    0.
  \end{gather*}

  In particular, if $s^\cS_m\leq o(m)$ as $m\to\infty$, then
  \begin{multline*}
    \PP_{\rn{x}\sim\mu^m}\Biggl[
      L_{\mu,F,\ell}\biggl(\rho^\cS\Bigl(\sigma^*\bigl(\rn{x},F^*_m(\rn{x})\bigr), \eta\Bigr)\biggr)
      -
      L^\alpha_{\rn{x}, F^*_m(\rn{x}), \ell}\biggl(\rho^\cS\Bigl(\sigma^*\bigl(\rn{x},F^*_m(\rn{x})\bigr), \eta\Bigr)\biggr)
      \geq
      \epsilon
      \Biggr]
    \\
    \leq
    \exp\left(
    -\frac{\epsilon^2}{2\cdot k^2\cdot\lVert\ell\rVert_\infty^2}\cdot m
    + \frac{\epsilon\cdot(\epsilon+\lVert\ell\rVert_\infty)}{\lVert\ell\rVert_\infty^2}\cdot o(m)
    \right),
  \end{multline*}
  where the $o(m)$ error term is as $m\to\infty$ with $k$ fixed, is uniform in $\epsilon$, and has the same uniformity in
  $\lVert\ell\rVert_\infty$ as the estimate $s^\cS_m\leq o(m)$.
\end{lemma}

\begin{proof}
  The result is trivial if $m<k$, so suppose $m\geq k$. By symmetry, we may also suppose without loss of generality that
  $\sigma\colon[s^\cS_m]\to[m]$ is such that $\im(\sigma)=\{m-s^\cS_m+1,m-s^\cS_m+2,\ldots,m\}$. Let us pick $\rn{x}$ at random
  according to $\mu^m$ and for ease of notation, we let
  \begin{align*}
    n & \df m - s^\cS_m, &
    \rn{H} & \df \rho^\cS_m\Bigl(\sigma^*\bigl(\rn{x},F^*_m(\rn{x})\bigr),\eta\Bigr), &
    \rn{x'} & \df \iota^*(\rn{x}),
  \end{align*}
  where $\iota\colon[n]\to[m]$ is the inclusion map (so $\rn{x'}$ is the random sample corresponding to the first $n$
  coordinates of $\rn{x}$, in their natural order). We also let $\alpha'\df\alpha\rest_{\binom{[n]}{k}}$ be the order choice for
  $[n]$ obtained by restricting $\alpha$ to the $k$-subsets of $[n]$.

  Note that
  \begin{align*}
    \MoveEqLeft
    L^\alpha_{\rn{x},F^*_m(\rn{x}),\ell}(\rn{H}) - L^{\alpha'}_{\rn{x'},F^*_n(\rn{x'}),\ell}(\rn{H})
    \\
    & =
    \begin{multlined}[t]
      \frac{1}{\binom{m}{k}}\cdot
      \sum_{U\in\binom{[m]}{k}}
      \ell\Bigl(\alpha_U^*(\rn{x}), b_\alpha\bigl(\rn{H}^*_m(\rn{x})\bigr)_U, b_\alpha\bigl(F^*_m(\rn{x})\bigr)_U\Bigr)
      \\
      -
      \frac{1}{\binom{n}{k}}\cdot
      \sum_{U\in\binom{[n]}{k}}
      \ell\Bigl(\alpha_U^*(\rn{x}), b_\alpha\bigl(\rn{H}^*_m(\rn{x})\bigr)_U, b_\alpha\bigl(F^*_m(\rn{x})\bigr)_U\Bigr)
    \end{multlined}
    \\
    & =
    \begin{multlined}[t]
      \left(\frac{1}{\binom{m}{k}} - \frac{1}{\binom{n}{k}}\right)\cdot
      \sum_{U\in\binom{[n]}{k}}
      \ell\Bigl(\alpha_U^*(\rn{x}), b_\alpha\bigl(\rn{H}^*_m(\rn{x})\bigr)_U, b_\alpha\bigl(F^*_m(\rn{x})\bigr)_U\Bigr)
      \\
      +
      \frac{1}{\binom{m}{k}}\cdot
      \sum_{U\in\binom{[m]}{k}\setminus\binom{[n]}{k}}
      \ell\Bigl(\alpha_U^*(\rn{x}), b_\alpha\bigl(\rn{H}^*_m(\rn{x})\bigr)_U, b_\alpha\bigl(F^*_m(\rn{x})\bigr)_U\Bigr),
    \end{multlined}
  \end{align*}
  and since $\ell$ is bounded and noting that the first term in the above is non-positive (as $n\leq m$) while the second is
  non-negative, we conclude that
  \begin{gather}\label{eq:nonpartconc:diff}
    \begin{aligned}
      \MoveEqLeft
      \bigl\lvert L^\alpha_{\rn{x},F^*_m(\rn{x}),\ell}(\rn{H}) - L^{\alpha'}_{\rn{x'},F^*_n(\rn{x'}),\ell}(\rn{H})\bigr\rvert
      \\
      & \leq
      \lVert\ell\rVert_\infty\cdot
      \max\left\{
      \left(\frac{1}{\binom{n}{k}} - \frac{1}{\binom{m}{k}}\right)\cdot\binom{n}{k},
      \frac{1}{\binom{m}{k}}\cdot\left(\binom{m}{k} - \binom{n}{k}\right)
      \right\}
      \\
      & =
      \left(1 - \frac{(n)_k}{(m)_k}\right)\cdot\lVert\ell\rVert_\infty.
    \end{aligned}
  \end{gather}

  Since we are mostly concerned with the setting in which $s^\cS_m\leq o(m)$ (or equivalently, $n = (1-o(1))\cdot m$), the
  difference above is small, it will suffice to compare $L_{\mu,F,\ell}(\rn{H})$ with
  $L^{\alpha'}_{\rn{x'},F^*_n(\rn{x'}),\ell}(\rn{H})$ (instead of comparing with $L^\alpha_{\rn{x},F^*_m(\rn{x}),\ell}(\rn{H})$
  directly).

  Let $\cF_0$ be the $\sigma$-algebra generated by $\rn{x}_{m-s^\cS_m+1},\ldots,\rn{x}_m$ (i.e., generated by
  $\sigma^*(\rn{x})$), so $\rn{H}$ is $\cF_0$-measurable. For each $i\in[m]$, let $\cF_i$ be the $\sigma$-algebra generated by
  $\cF_0$ and $(\rn{x}_1,\rn{x}_2,\ldots,\rn{x}_i)$ and let
  $\rn{Z}_i\df\EE[L^{\alpha'}_{\rn{x'},F^*_n(\rn{x'}),\ell}(\rn{H})\given\cF_i]$.

  We note that $(\rn{Z}_i)_{i=0}^n$ forms a (Doob) martingale with respect to $(\cF_i)_{i=0}^n$:
  \begin{gather*}
    \EE[\rn{Z}_{i+1}\given\cF_i]
    =
    \EE\Bigl[
      \EE\bigl[L^{\alpha'}_{\rn{x'},F^*_n(\rn{x'}),\ell}(\rn{H})\given[\bigm]\cF_{i+1}\bigr]
      \given[\Bigm]
      \cF_i
      \Bigr]
    =
    \EE\bigl[L^{\alpha'}_{\rn{x'},F^*_n(\rn{x'}),\ell}(\rn{H})\given[\bigm]\cF_i\bigr]
    =
    \rn{Z}_i.
  \end{gather*}
  We also note that
  \begin{align*}
    \rn{Z}_0
    & =
    \EE\bigl[L^{\alpha'}_{\rn{x'},F^*_n(\rn{x'}),\ell}(\rn{H})\given[\bigm]\cF_0\bigr]
    =
    L_{\mu,F,\ell}(\rn{H}),
    \\
    \rn{Z}_n
    & =
    \EE\bigl[L^{\alpha'}_{\rn{x'},F^*_n(\rn{x'}),\ell}(\rn{H})\given[\bigm]\cF_n\bigr]
    =
    L^{\alpha'}_{\rn{x'},F^*_n(\rn{x'}),\ell}(\rn{H}),
  \end{align*}
  where the former is by linearity of conditional expectation (and the fact that $\rn{H}$ is $\cF_0$-measurable) and the latter
  is simply because the underlying random variable is $\cF_n$-measurable.

  We will show that $\rn{Z}_0$ and $\rn{Z}_n$ are close with high (conditional) probability using Azuma's Inequality. Toward
  that, let us bound $\lvert\rn{Z}_i - \rn{Z}_{i-1}\rvert$ for each $i\in[n]$. By linearity of conditional expectation, we have
  \begin{multline}\label{eq:nonpartconc:martdiff}
    \lvert\rn{Z}_i - \rn{Z}_{i-1}\rvert
    =
    \bigggl\lvert
    \frac{1}{\binom{n}{k}}\cdot
    \sum_{U\in\binom{[n]}{k}}
    \Biggl(
    \EE\biggl[
      \ell\Bigl(\alpha^*_U(\rn{x}), b_\alpha\bigl(\rn{H}^*_m(x)\bigr)_U, b_\alpha\bigl(F^*_m(x)\bigr)_U\Bigr)
      \given[\biggm]
      \cF_i
      \biggr]
    \\
    -
    \EE\biggl[
      \ell\Bigl(\alpha^*_U(\rn{x}), b_\alpha\bigl(\rn{H}^*_m(x)\bigr)_U, b_\alpha\bigl(F^*_m(x)\bigr)_U\Bigr)
      \given[\biggm]
      \cF_{i+1}
      \biggr]
    \Biggr)
    \bigggr\rvert.
  \end{multline}

  We now argue that most of the terms in the above cancel out. Namely, consider a term corresponding to $U\in\binom{[n]}{k}$
  such that $i\notin U$, then
  \begin{gather*}
    \ell\Bigl(\alpha^*_U(\rn{x}), b_\alpha\bigl(\rn{H}^*_m(x)\bigr)_U, b_\alpha\bigl(F^*_m(x)\bigr)_U\Bigr)
  \end{gather*}
  is independent from $\rn{x}_i$; in fact, since $(\rn{x}_j)_{j=1}^m$ is mutually independent, the above is conditionally
  independent from $\rn{x}_i$ given $\cF_{i-1}$, so the above gives the same value when conditioned on $\cF_{i-1}$ as it does
  when conditioned on $\cF_i$. Thus, whenever $i\notin U$, the corresponding term in~\eqref{eq:nonpartconc:martdiff} cancels
  out.

  As there are exactly $\binom{n-1}{k-1}$ many $U\in\binom{[n]}{k}$ with $i\in U$, using triangle inequality and the fact that
  $\ell$ is bounded, we conclude that
  \begin{gather*}
    \lvert\rn{Z}_i - \rn{Z}_{i-1}\rvert
    \leq
    \frac{\binom{n-1}{k-1}}{\binom{n}{k}}\cdot\lVert\ell\rVert_\infty
    =
    \frac{k}{n}\cdot\lVert\ell\rVert_\infty.
  \end{gather*}
  We now note that since
  \begin{gather*}
    \widetilde{\epsilon}
    \df
    \epsilon - \left(1 - \frac{(m-s^\cS_m)_k}{(m)_k}\right)\cdot\lVert\ell\rVert_\infty
    =
    \epsilon - \left(1 - \frac{(n)_k}{(m)_k}\right)\cdot\lVert\ell\rVert_\infty,
  \end{gather*}
  the condition~\eqref{eq:nonpartconc:condition} implies $\widetilde{\epsilon} > 0$, so by Azuma's Inequality (conditioned on
  $\cF_0$), we get
  \begin{align*}
    \PP_{\rn{x}\sim\mu^m}\biggl[
      L_{\mu,F,\ell}(\rn{H}) - L^{\alpha'}_{\rn{x'},F^*_n(\rn{x'}),\ell}(\rn{H})
      \geq
      \widetilde{\epsilon}
      \biggr]    
    & =
    \PP[\rn{Z}_n - \rn{Z}_0\leq -\widetilde{\epsilon}]
    \\
    & \leq
    \exp\left(-\frac{\widetilde{\epsilon}^2}{2\cdot\sum_{i=1}^n(k\cdot\lVert\ell\rVert_\infty/n)^2}\right)
    \\
    & =
    \exp\left(-\frac{\widetilde{\epsilon}^2\cdot n}{2\cdot k^2\cdot\lVert\ell\rVert_\infty^2}\right)
    \\
    & =
    \exp\left(-\frac{\widetilde{\epsilon}^2\cdot (m-s^\cS_m)}{2\cdot k^2\cdot\lVert\ell\rVert_\infty^2}\right).
  \end{align*}

  Putting this together with~\eqref{eq:nonpartconc:diff}, we conclude that
  \begin{gather*}
    \PP_{\rn{x}\sim\mu^m}\bigl[
      L_{\mu,F,\ell}(\rn{H}) - L^\alpha_{\rn{x},F^*_n(\rn{x}),\ell}(\rn{H})
      \geq
      \epsilon
      \bigr]
    \leq
    \exp\left(-\frac{\widetilde{\epsilon}^2\cdot (m-s^\cS_m)}{2\cdot k^2\cdot\lVert\ell\rVert_\infty^2}\right),
  \end{gather*}
  as desired.

  \medskip

  Let us now estimate the above asymptotically when $s^\cS_m\leq o(m)$. First note that
  \begin{gather*}
    0
    \leq
    1 - \frac{(m-s^\cS_m)_k}{(m)_k}
    \leq
    1 - \left(1 - \frac{s^\cS_m}{m}\right)^k
    \leq
    o(1)
  \end{gather*}
  as $m\to\infty$ with $k$ fixed. Since we also have $m - s^\cS_m = (1 - o(1))\cdot m$, we conclude that
  \begin{align*}
    \exp\left(-\frac{\widetilde{\epsilon}^2\cdot (m-s^\cS_m)}{2\cdot k^2\cdot\lVert\ell\rVert_\infty^2}\right)
    & =
    \exp\left(
    -\frac{(\epsilon - \lVert\ell\rVert_\infty\cdot o(1))^2\cdot(1-o(1))\cdot m}{2\cdot k^2\cdot\lVert\ell\rVert_\infty^2}
    \right)
    \\
    & \leq
    \exp\left(
    -\frac{\epsilon^2}{2\cdot k^2\cdot\lVert\ell\rVert_\infty^2}\cdot m
    +\frac{\epsilon\cdot(\epsilon+\lVert\ell\rVert_\infty)}{\lVert\ell\rVert_\infty^2}\cdot o(m)
    \right),
  \end{align*}
  where the $o(m)$ error term is as $m\to\infty$ with $k$ fixed, is uniform in $\epsilon$, and has the same uniformity in
  $\lVert\ell\rVert_\infty$ as the estimate $s^\cS_m\leq o(m)$.
\end{proof}

The next theorem shows that every $k$-ary hypothesis class that admits a $k$-ary compression scheme $\cS$ of essentially
non-trivial quality must be $k$-PAC learnable. More specifically, a trivial $k$-ary compression scheme has $s^\cS_m=m$ and
$h^\cS_m=1$, which yields
\begin{gather*}
  \ln(h^\cS_m) + s^\cS_m\cdot\ln(m) = m\cdot\ln(m)
\end{gather*}
and our assumption is that we have a $k$-ary compression scheme in which the quantity above is $o(m)$ instead. The main
differences of the next theorem to its partite counterpart, Theorem~\ref{thm:partSC->PAC}, stem from the fact that the bounds of
Lemma~\ref{lem:nonpartconc} are different from those in Lemma~\ref{lem:partconc} and the union bound involves a different number
of events.

\begin{theorem}\label{thm:nonpartSC->PAC}
  Let $\Omega=(X,\cB)$ and $\Lambda=(Y,\cB')$ be non-empty standard Borel spaces with $\lvert Y\rvert\geq 2$, let $k\in\NN_+$,
  let $\ell\colon X^k\times Y^{S_k}\times Y^{S_k}\to\RR_{\geq 0}$ be a bounded $k$-ary loss function and let $\cS$ be a $k$-ary
  compression scheme for a $k$-ary hypothesis class $\cH\subseteq\cF_k(\Omega,\Lambda)$ such that
  \begin{gather}\label{eq:nonpartSC->PAC:condition}
    \ln(h^\cS_m) + s^\cS_m\cdot\ln(m) \leq o(m)
  \end{gather}
  as $m\to\infty$ (which in particular follows from $b_\cS(m)\leq o(m/\ln(m))$ when $Y$ is finite), then $\cH$ is $k$-PAC
  learnable with respect to $\ell$. More specifically, letting $\cA$ be given by
  \begin{gather*}
    \cA(x,y)
    \df
    \rho^\cS_m\bigl(\kappa^\cS_m(x,y)\bigr)
    \qquad
    (m\in\NN, x\in X^m, y\in Y^{([m])_k})
  \end{gather*}
  gives a $k$-PAC learner for $\cH$ with respect to $\ell$ with learning guarantee
  \begin{multline}\label{eq:nonpartSC->PAC:mPAC}
    m^{\PAC}_{\cH,\ell,\cA}(\epsilon,\delta)
    \df
    \min\Biggl\{m_0\in\NN_+
    \;\Biggm\vert\;
    \forall m\geq m_0,
    \left(1 - \frac{(m-s^\cS_m)_k}{(m)_k}\right)\cdot\lVert\ell\rVert_\infty < \epsilon
    \\
    \land
    (m)_{s^\cS_m}\cdot h^\cS_m\cdot
    \exp\left(
    -\frac{
      (\epsilon - (1-(m-s^\cS_m)_k/(m)_k)\cdot\lVert\ell\rVert_\infty)^2\cdot(m-s^\cS_m)
    }{
      2\cdot k^2\cdot\lVert\ell\rVert_\infty^2
    }
    \right)
    \leq
    \delta
    \Biggr\}.
  \end{multline}

  In particular, if
  \begin{gather}\label{eq:nonpartSC->PAC:strongcondition}
    \ln(h^\cS_m) + s^\cS_m\cdot\ln(m)
    \leq
    \bigl(1 + o(1)\bigr)\cdot
    \sqrt{2\cdot m}\cdot\ln(m),
  \end{gather}
  as $m\to\infty$, then
  \begin{gather*}
    m^{\PAC}_{\cH,\ell,\cA}(\epsilon,\delta)
    \leq
    \bigl(1 + o(1)\bigr)\cdot
    \frac{2\cdot k^2\cdot\lVert\ell\rVert_\infty^2}{\epsilon^2}\cdot
    \max\left\{1, \ln\left(\frac{1}{\delta}\right)\right\},
  \end{gather*}
  where the $o(1)$ error term is as $\epsilon\to 0$ with $k$ fixed, is uniform in $\delta$, and has the same uniformity in
  $\lVert\ell\rVert_\infty$ as the estimate $o(1)$ in~\eqref{eq:nonpartSC->PAC:strongcondition} provided
  $\lVert\ell\rVert_\infty$ is bounded away from $0$.
\end{theorem}

\begin{proof}
  Let us start by showing that $b_\cS(m)\leq o(m/\ln(m))$ (when $Y$ is finite) implies the
  condition~\eqref{eq:nonpartSC->PAC:condition} (i.e., $\ln(h^\cS_m) + s^\cS_m\cdot\ln(m)\leq o(m)$). Indeed, for $m\geq e$, we
  have
  \begin{gather*}
    \ln(h^\cS_m) + s^\cS_m\cdot\ln(m)
    \leq
    \ln(m)\cdot\bigl(\log_2(h^\cS_m) + s^\cS_m\cdot\log_2\lvert Y\rvert\bigr)
    \leq
    \ln(m)\cdot b_\cS(m)
    \leq
    o(m).
  \end{gather*}

  Let us now argue that the minimum in~\eqref{eq:nonpartSC->PAC:mPAC} indeed exists. For this, we simply need to argue that there
  exists $m_0\in\NN_+$ satisfying both conditions in the definition of the minimum. But indeed, since we are assuming
  that~\eqref{eq:nonpartSC->PAC:condition} (i.e., $\ln(h^\cS_m) + s^\cS_m\cdot\ln(m)\leq o(m)$) holds, it follows that both
  \begin{gather*}
    \left(1 - \frac{(m-s^\cS_m)_k}{(m)_k}\right)\cdot\lVert\ell\rVert_\infty,
    \\
    (m)_{s^\cS_m}\cdot h^\cS_m\cdot
    \exp\left(
    -\frac{
      (\epsilon - (1-(m-s^\cS_m)_k/(m)_k)\cdot\lVert\ell\rVert_\infty)^2\cdot(m-s^\cS_m)
    }{
      2\cdot k^2\cdot\lVert\ell\rVert_\infty^2
    }
    \right)
  \end{gather*}
  tend to $0$ as $m\to\infty$ (with all other parameters fixed). So there must exist some $m_0\in\NN$ large enough so that for
  every $m\geq m_0$, the first expression above is less than $\epsilon$ and the second expression above is at most $\delta$.

  \medskip

  We now prove the result. Let $\mu\in\Pr(\Omega)$ be a probability measure and $F\in\cF_k(\Omega,\Lambda)$ be a $k$-ary
  hypothesis that is realizable with respect to $\mu$ and $\ell$. Given $\epsilon,\delta\in(0,1)$, let $m\geq
  m^{\PAC}_{\cH,\ell,\cA}(\epsilon,\delta)$ be an integer, let $\alpha$ be an order choice for $[m]$ and let
  $\widetilde{\epsilon}$ be defined in terms of $m$, $s^\cS_m$, $k$ and $\ell$ as in Lemma~\ref{lem:nonpartconc} and note that
  since $m\geq m^{\PAC}_{\cH,\ell,\cA}(\epsilon,\delta)$, the first condition in the definition of the latter implies that the
  condition~\eqref{eq:nonpartconc:condition} of Lemma~\ref{lem:nonpartconc} is satisfied.

  We now pick $\rn{x}$ at random according to $\mu^m$ and for each $\sigma\in([m])_{s^\cS_m}$ and each $\eta\in[h^\cS_m]$, we
  let $E_{\sigma,\eta}(\rn{x})$ be the event
  \begin{gather*}
    L_{\mu,F,\ell}\biggl(\rho^\cS_m\Bigl(\sigma^*\bigl(\rn{x}, F^*_m(\rn{x})\bigr),\eta\Bigr)\biggr)
    -
    L_{\rn{x},F^*_m(\rn{x}),\ell}^\alpha\biggl(\rho^\cS_m\Bigl(\sigma^*\bigl(\rn{x}, F^*_m(\rn{x})\bigr),\eta\Bigr)\biggr)
    \geq
    \epsilon
  \end{gather*}
  so that Lemma~\ref{lem:nonpartconc} says
  \begin{gather*}
    \PP_{\rn{x}\sim\mu^m}\bigl[E_{\sigma,\eta}(\rn{x})\bigr]
    \leq
    \exp\left(-\frac{\widetilde{\epsilon}^2\cdot(m-s^\cS_m)}{2\cdot k^2\cdot\lVert\ell\rVert_\infty^2}\right).
  \end{gather*}

  If we let $E(\rn{x})$ be the disjunction of the events $E_{\sigma,\eta}(\rn{x})$, since there are $(m)_{s^\cS_m}$ many such
  $\sigma$ and $h^\cS_m$ many such $\eta$, by the union bound, we get
  \begin{gather*}
    \PP_{\rn{x}\sim\mu^m}\bigl[E(\rn{x})\bigr]
    \leq
    (m)_{s^\cS_m}\cdot h^\cS_m\cdot
    \exp\left(-\frac{\widetilde{\epsilon}^2\cdot(m-s^\cS_m)}{2\cdot k^2\cdot\lVert\ell\rVert_\infty^2}\right).
  \end{gather*}

  Note also that we have the compression scheme guarantee that
  \begin{gather*}
    L^\alpha_{\rn{x},F^*_m(\rn{x}),\ell}\biggl(\rho^\cS_m\Bigl(\kappa^\cS_m\bigl(\rn{x}, F^*_m(\rn{x})\bigr)\Bigr)\biggr)
    =
    0
  \end{gather*}
  whenever $(\rn{x},F^*_m(\rn{x}))$ is sample realizable in $\cH$ with respect to $\ell$ and $\alpha$, which happens with
  probability $1$ as $F$ is realizable in $\cH$ with respect to $\ell$.

  The above along with the fact that the (random) injection $\sigma^\cS_m(\rn{x},F^*_m(\rn{x}))$ and the (random) header
  $\eta^\cS_m(\rn{x},F^*_m(\rn{x}))$ are elements of $([m])_{s^\cS_m}$ and $[h^\cS_m]$, respectively, implies that
  \begin{align*}
    \MoveEqLeft
    \PP_{\rn{x}\sim\mu^m}\biggl[
        L_{\mu,F,\ell}\Bigl(\cA\bigl(\rn{x}, F^*_m(\rn{x})\bigr)\Bigr) > \epsilon
        \biggr]
    \\
    & =
    \PP_{\rn{x}\sim\mu^m}\Biggl[
      L_{\mu,F,\ell}\biggl(\rho^\cS_m\Bigl(\kappa^\cS_m\bigl(\rn{x}, F^*_m(\rn{x})\bigr)\Bigr)\biggr) > \epsilon
      \Biggr]
    \\
    & \leq
    \PP_{\rn{x}\sim\mu^m}\bigl[E(\rn{x})\bigr]
    \\
    & \leq
    (m)_{s^\cS_m}\cdot h^\cS_m\cdot
    \exp\left(-\frac{\widetilde{\epsilon}^2\cdot(m-s^\cS_m)}{2\cdot k^2\cdot\lVert\ell\rVert_\infty^2}\right)
    \\
    & =
    (m)_{s^\cS_m}\cdot h^\cS_m\cdot
    \exp\left(
    -\frac{
      (\epsilon - (1-(m-s^\cS_m)_k/(m)_k)\cdot\lVert\ell\rVert_\infty)^2\cdot(m-s^\cS_m)
    }{
      2\cdot k^2\cdot\lVert\ell\rVert_\infty^2
    }
    \right)
    \\
    & \leq
    \delta,
  \end{align*}
  where the last inequality follows since $m\geq m^{\PAC}_{\cH,\ell,\cA}(\epsilon,\delta)$ and using the second condition in the
  definition of the latter. Thus $\cA$ is a $k$-PAC learner for $\cH$ with respect to $\ell$.

  \medskip

  It remains to make the asymptotic estimate of $m^{\PAC}_{\cH,\ell,\cA}(\epsilon,\delta)$ when the stronger
  condition~\eqref{eq:nonpartSC->PAC:strongcondition} (i.e., $\ln(h^\cS_m) + s^\cS_m\cdot\ln(m) \leq (1+o(1))\cdot\sqrt{2\cdot
    m}\cdot\ln(m)$) holds. (It is clear that the condition~\eqref{eq:nonpartSC->PAC:strongcondition} implies the
  condition~\eqref{eq:nonpartSC->PAC:condition} that $\ln(h^\cS_m) + s^\cS_m\cdot\ln(m)\leq o(m)$.)

  Let $m_1(\epsilon)\in\NN_+$ be the least integer such that for every $m\geq m_1(\epsilon)$, we have
  \begin{gather}\label{eq:nonpartSC->PAC:m1def}
    \left(1 - \frac{(m-s^\cS_m)_k}{(m)_k}\right)\cdot\lVert\ell\rVert_\infty < \epsilon
  \end{gather}
  and let $m_2(\epsilon,\delta)\in\NN_+$ be the least integer such that for every $m\geq m_2(\epsilon,\delta)$, we have
  \begin{gather}\label{eq:nonpartSC->PAC:m2def}
    (m)_{s^\cS_m}\cdot h^\cS_m\cdot
    \exp\left(
    -\frac{
      (\epsilon - (1-(m-s^\cS_m)_k/(m)_k)\cdot\lVert\ell\rVert_\infty)^2\cdot(m-s^\cS_m)
    }{
      2\cdot k^2\cdot\lVert\ell\rVert_\infty^2
    }
    \right)
    \leq
    \delta.
  \end{gather}
  Since $m^{\PAC}_{\cH,\ell,\cA}(\epsilon,\delta)=\max\{m_1(\epsilon),m_2(\epsilon,\delta)\}$, it suffices to give asymptotic
  estimates for $m_1(\epsilon)$ and $m_2(\epsilon,\delta)$.

  To estimate $m_1(\epsilon)$, note that if $\lVert\ell\rVert_\infty < \epsilon$, then $m_1(\epsilon) = k$, so suppose that
  $\lVert\ell\rVert_\infty\geq\epsilon$ and note that
  \begin{gather}\label{eq:nonpartSC->PAC:towardm1}
    \begin{aligned}
      0
      & \leq
      1 - \frac{(m-s^\cS_m)_k}{(m)_k}
      \leq
      1 - \left(1 - \frac{s^\cS_m+k}{m}\right)^k
      \\
      & \leq
      1 - \left(1 - \bigl(1 + o(1)\bigr)\cdot\sqrt{\frac{2}{m}}\right)^k
      =
      \bigl(1 + o(1)\bigr)\cdot k\cdot\sqrt{\frac{2}{m}}
    \end{aligned}
  \end{gather}
  as $m\to\infty$ with $k$ fixed.

  Plugging this asymptotic estimate in the definition~\eqref{eq:nonpartSC->PAC:m1def} of $m_1(\epsilon)$, we conclude that
  \begin{gather}\label{eq:nonpartSC->PAC:m1estimate}
    m_1(\epsilon)
    \leq
    \bigl(1 + o(1)\bigr)\cdot
    \frac{2\cdot k^2\cdot\lVert\ell\rVert_\infty^2}{\epsilon^2},
  \end{gather}
  where the $o(1)$ error term is as $\epsilon\to 0$ with $k$ fixed, is uniform in $\delta$ (as $m_1(\epsilon)$ does not depend
  on $\delta$ at all), and has the same uniformity in $\lVert\ell\rVert_\infty$ as the estimate $o(1)$
  in~\eqref{eq:nonpartSC->PAC:strongcondition} provided $\lVert\ell\rVert_\infty$ is bounded away from $0$ (this is to cover the case
  $\lVert\ell\rVert_\infty < \epsilon$ when $m_1(\epsilon) = k$).

  To estimate $m_2(\epsilon,\delta)$, we use~\eqref{eq:nonpartSC->PAC:towardm1} along with the
  condition~\eqref{eq:nonpartSC->PAC:condition} (i.e., $\ln(h^\cS_m) + s^\cS_m\cdot\ln(m)\leq o(m)$) to obtain
  \begin{align*}
    \MoveEqLeft
    (m)_{s^\cS_m}\cdot h^\cS_m\cdot
    \exp\left(
    -\frac{
      (\epsilon - (1-(m-s^\cS_m)_k/(m)_k)\cdot\lVert\ell\rVert_\infty)^2\cdot(m-s^\cS_m)
    }{
      2\cdot k^2\cdot\lVert\ell\rVert_\infty^2
    }
    \right)
    \\
    & \leq
    \exp\left(
    -\frac{(\epsilon-o(1)\cdot\lVert\ell\rVert_\infty)^2\cdot(1+o(1))\cdot m}{2\cdot k^2\cdot\lVert\ell\rVert_\infty^2}
    + o(m)
    \right)
    \\
    & \leq
    \exp\left(
    -\left(\frac{\epsilon^2}{2\cdot k^2\cdot\lVert\ell\rVert_\infty^2} - o(1)\right)\cdot m
    \right),
  \end{align*}
  where the error terms are as $m\to\infty$ with $k$ fixed, are uniform in $\epsilon$, and have the same uniformity in
  $\lVert\ell\rVert_\infty$ as the estimate $o(1)$ in~\eqref{eq:nonpartSC->PAC:strongcondition}.

  Plugging this asymptotic estimate in the definition~\eqref{eq:nonpartSC->PAC:m2def} of $m_2(\epsilon,\delta)$, we conclude that
  \begin{gather*}
    m_2(\epsilon,\delta)
    \leq
    \left(\frac{2\cdot k^2\cdot\lVert\ell\rVert^2}{\epsilon^2} + o(1)\right)\cdot\ln\left(\frac{1}{\delta}\right),
  \end{gather*}
  where the $o(1)$ error term is as $\epsilon\to 0$, is uniform in $\delta$, and has the same uniformity in
  $\lVert\ell\rVert_\infty$ as the estimate $o(1)$ in~\eqref{eq:nonpartSC->PAC:strongcondition}.

  Putting the estimate above with our estimate for $m_1(\epsilon)$ in~\eqref{eq:nonpartSC->PAC:m1estimate}, we get
  \begin{gather*}
    m^{\PAC}_{\cH,\ell,\cA}(\epsilon,\delta)
    \leq
    \bigl(1 + o(1)\bigr)\cdot
    \frac{2\cdot k^2\cdot\lVert\ell\rVert_\infty^2}{\epsilon^2}\cdot\max\left\{1, \ln\left(\frac{1}{\delta}\right)\right\}
  \end{gather*}
  as desired since $m^{\PAC}_{\cH,\ell,\cA} = \max\{m_1(\epsilon),m_2(\epsilon,\delta)\}$.
\end{proof}

\section{Discussion and final remarks}

Let us conclude the paper by briefly addressing what was not covered in this work.

\subsection{What about the other direction?}
\label{subsec:converse}

First and foremost, in the classical setting, David--Moran--Yehudayoff~\cite{DMY16} proved that sample compression is in fact
equivalent to PAC learnability\footnote{Their proof is when $Y=\{0,1\}$ and with $0/1$-loss, but it readily adapts to the case
when $Y$ is finite and $\ell$ is \emph{separated} in the sense that it assigns $0$ penalty to correct guesses and is bounded
away from $0$ in incorrect guesses.}, that is, given a PAC learner $\cA$, one can produce from it a sample compression scheme.
In a very high-level, their proof follows by first considering a two-player game over a given sample in which a learner attempts
to pick a hypothesis $H$ that is the output of $\cA$ on a subsample of size $m^{\PAC}(1/3,1/3)$ while a spoiler tries to pick a
point of the sample on which $H$ has large loss. The PAC learning guarantee says that if the spoiler plays a mixed strategy
first, then the learner can ensure small loss even with a pure strategy playing second. By von~Neumann's Minimax Theorem, there
must be a mixed strategy $S$ that the learner can play first that ensures that any pure strategy of the spoiler playing second
yields small loss. This then allows us to compress the sample by considering a large enough sample from the mixed strategy $S$
and taking a majority vote.

The most natural way to adapt the proof above to high-arity is in the partite setting. However the issue here is that since in
high-arity we only allow product distributions over $X_1\times\cdots\times X_k$, von~Neumann's Minimax Theorem no longer applies
(there is no reason to believe that optimal mixed strategies must necessarily come from product distributions). An alternative
would be to setup the Minimax as a game with one learner and $k$ spoilers, but even though the corresponding cost function is
multilinear in the probability distributions, there is no reason to believe (nor is it a reasonable assumption to make) that we
have the correct (quasi-)concavities/convexities. The non-partite setting is even worse: since it requires product
distributions, the $k$ spoilers would then have to play the same mixed strategy.

\subsection{Approximate/agnostic compression schemes and higher-order variables}
\label{subsec:approxag}

In~\cite{DMY16}, David--Moran--Yehudayoff also defined the notions of approximate and approximate agnostic sample compression
schemes. In an approximate compression scheme, instead of assuming that the reconstruction process yields zero loss, one assumes
that it gives small loss, that is, for every realizable sample $(x,y)$ of size $m$, we want
$L_{x,y,\ell}(\rho_m(\kappa_m(x,y)))\leq\epsilon_m$ where $\epsilon_m > 0$ (and typically thought to satisfy $\epsilon_m\to 0$
as $m\to\infty$). For approximate agnostic compression schemes, one drops the realizability assumption but requires instead that
the reconstruction is competitive in the sense that
\begin{gather*}
  L_{x,y,\ell}(\rho_m(\kappa_m(x,y)))\leq \inf_{H\in\cH} L_{x,y,\ell}(H) + \epsilon_m.
\end{gather*}
As was observed in~\cite{DMY16}, the fact that (agnostic, respectively) PAC learnability implies approximate (agnostic,
respectively) sample compressibility is very simple: one simply considers the uniform measure on the sample, uses the PAC
guarantee to claim that there exists a small subsample that when fed to the algorithm yields small loss (or is competitive in
the agnostic case), then uses a bit of extra information to encode potential repetitions. In the other direction, the argument
that sample compressibility implies PAC learnability easily adapts to the approximate (agnostic) case.

Of course, the definitions of approximate and approximate agnostic sample compressibility readily adapt to the high-arity
setting (both partite and non-partite). It is straightforward to check that the proof of equivalence between approximate sample
compressibility and PAC learnability lifts to the high-arity case (both partite and non-partite). In the remainder of this
section, we sketch the equivalence of the agnostic case.

First, what is high-arity agnostic PAC learnability? In the definition of~\cite{CM24} and in analogy with classical PAC theory,
there are two changes: the adversary is allowed to play an ``agnostic'' distribution and the learner is only tasked with being
competitive.

The key technicality is in what an ``agnostic'' distribution is. In the non-partite, the adversary picks a distribution $\nu$
over $X_1^{\NN_+}\times\cdots\times X_k^{\NN_+}\times Y^{\NN_+^k}$ (we want to think of elements of this set as infinite
$k$-dimensional tensors with names on each of the indices) that satisfies:
\begin{description}
\item[Separately exchangeability] in the sense that $\nu$ is invariant under the natural right-action of $S_{\NN_+}^k$ (the
  $i$th copy of $S_{\NN_+}$ permutes the coordinates of $X_i^{\NN_+}$ and $S_{\NN_+}^k$ as a whole permutes the coordinates of
  $Y^{\NN_+}$).
\item[Locality] in the sense that marginals on $U_1\times\cdots\times U_k$ and $V_1\times\cdots\times V_k$ are independent
  provided $U_i\cap V_i=\varnothing$ for every $i\in[k]$.
\end{description}
A labeled $m$ sample is then generated by taking the marginal of $\nu$ on $[m]^k$.

In the non-partite, the adversary picks a distribution $\nu$ over $(X^{\NN_+})^k\times Y^{(\NN_+)_k}$ that satisfies:
\begin{description}
\item[(Joint) Exchangeability] in the sense that $\nu$ is invariant under the natural right-action of $S_{\NN_+}$.
\item[Locality] in the sense that marginals on $U^k$ and $V^k$ are independent provided $U\cap V=\varnothing$.
\end{description}
A labeled $m$ sample is then generated by taking the marginal of $\nu$ on $[m]^k$.

One then gets a handle on these distributions by invoking the Aldous--Hoover Theorem~\cite{Hoo79,Ald85,Ald81} from
exchangeability theory (see also~\cite[\S 7]{Kal05} for a more modern presentation) to get a normal representation of such
distributions. Without getting into the (highly) technical details, let us mention that the proof presented here readily adapts
to this more general setting simply because using this Aldous--Hoover representation, one can check that the martingales defined
in the proofs remain martingales (with the same bounds on martingale differences), which allow the key application of
(two-sided) Azuma's Inequality. In fact, similar martingales and an application of Azuma's Inequality were used
in~\cite[Lemma~7.3]{CM24}.

\bibliographystyle{alpha}
\bibliography{refs}

\end{document}